\documentclass[10pt]{amsart}
\usepackage{amssymb,mathrsfs,graphicx,extpfeil}
\usepackage{epsfig}
\usepackage{indentfirst, latexsym, amssymb, enumerate,amsmath,graphicx,amsthm}
\usepackage{float}
\usepackage{colortbl}
\usepackage{epsfig,subfigure}
\usepackage[numbers,sort&compress]{natbib}

\title[Deterministic Dynamics of Score-Based Models]{Deterministic Dynamics of Sampling Processes \\in Score-Based Diffusion Models with Multiplicative Noise Conditioning}

\author[Kim]{Doheon Kim}
\address[Doheon Kim]{\newline Department of Mathematical Data Science, Hanyang University, ERICA Campus,\newline Hanyangdaehak-ro 55, Sangnok-gu, Ansan, Gyeonggi-do 15588, Republic of Korea}
\email{doheonkim@hanyang.ac.kr}

\newtheorem{theorem}{Theorem}[section]
\newtheorem{lemma}{Lemma}[section]

\newtheorem{proposition}{Proposition}[section]

\newtheorem{remark}{Remark}[section]
\newtheorem{assumption}{Assumption}[section]

\newtheorem{definition}{Definition}[section]

\allowdisplaybreaks

\def\charf {\mbox{{\text 1}\kern-.30em {\text l}}}

\begin{document}


%
\allowdisplaybreaks

\begin{abstract}
Score-based diffusion models generate new samples by learning the score function associated with a diffusion process. While the effectiveness of these models can be theoretically explained using differential equations related to the sampling process, previous work by Song and Ermon (2020) demonstrated that neural networks using multiplicative noise conditioning can still generate satisfactory samples. In this setup, the model is expressed as the product of two functions: one depending on the spatial variable and the other on the noise magnitude. This structure limits the model's ability to represent a more general relationship between the spatial variable and the noise, indicating that it cannot fully learn the correct score. Despite this limitation, the models perform well in practice. In this work, we provide a theoretical explanation for this phenomenon by studying the deterministic dynamics of the associated differential equations, offering insight into how the model operates.
\end{abstract}

\maketitle \centerline{\date}

%

\section{Introduction}
Diffusion models gradually convert training data into noise, and learn the reverse process to generate new samples from noise. There are three major formulations of these models \cite{YZSHXZZCY}: denoising diffusion probabilistic models \cite{HJA, SWMG}, which formulate the forward and backward diffusions as Markov chains; score-based generative models \cite{SE19, SE20}, which learn the score function associated with the forward diffusion to create the reverse process; and stochastic differential equations (SDEs) \cite{SDME, SSKKEP}, which unify the aforementioned two approaches. In this paper, we concentrate on a specific neural network architecture proposed in \cite{SE20}, examine the necessity for a new theory to elucidate its effectiveness, and present such a theory. 

The term {\it score} in \cite{SE19, SE20} refers to the gradient of the log probability density \cite{Hyvarinen, LLJ}. The model proposed in \cite{SE19} introduces a sequence of noise $\{\sigma_t\}_t$ with increasing magnitudes to the data to get a sequence of density functions $\{p_{\sigma_t}\}_t$, and learns the associated score function $\nabla \log p_{\sigma_t}(x)$ for each magnitude using a single neural network $s_\theta(\sigma_t,x)$, called a {\it noise conditional score network.} In this model, the training objective is a weighted sum of denoising score matching objectives \cite{Vincent}, with the weights depending on the noise magnitudes. In \cite{SSKKEP}, the sequence of noise and the weighted sum were reinterpreted as discretizations of time-varying noise and an integral with respect to time, respectively, and the latter was proposed as a new training objective. Both discrete and continuous training objectives have been shown to be equivalent to certain forms of distance between the score functions and the models being trained.

In \cite{SE20}, several techniques were suggested to improve the model proposed in \cite{SE19}. One of these techniques involves incorporating noise information by dividing the output of a noise-independent neural network by the noise magnitude. In other words, the model $s_\theta(x)/\sigma_t$ is trained to approximate $\nabla \log p_{\sigma_t}(x)$. In such cases, even if the neural network $s_\theta$ has ample capacity, it is unrealistic to anticipate that the training objective proposed in either \cite{SE19} or \cite{SSKKEP} would yield correct score functions, owing to the simplistic integration of noise information. The main result of this paper is the development of a theory explaining the effectiveness of sampling processes driven by the model with multiplicative noise conditioning.

In Section \ref{sec-2}, we provide an explicit expression of the optimal model according to the training objective in \cite{SE20}. In other words, we interpret the training objective as a certain form of distance between the optimal model and the model being trained. Unsurprisingly, the optimal model differs from the score function.  We adopt the continuous-time formulation of the training objective from \cite{SDME, SSKKEP}, and find the optimal model resulting from the following objective:
\begin{equation}\label{generalloss}
\underset{\theta}{\operatorname{minimize}}
\int_\varepsilon^T\int_{\mathbb R^d}\int_{\mathbb R^d}\bigg|\frac{1}{\sigma_t^{\alpha}}s_\theta( x)-\nabla \log p_{\sigma_t}(  x|\tilde x)\bigg|^2  p_{\sigma_t}(  x|\tilde x)d xd\mu_\mathrm{data}(\tilde x)\lambda(t)dt.
\end{equation}
Here $\sigma_t$ represents the amount of noise injected up to time $t$, $p_{\sigma_t}(\cdot|\tilde x)$ is the probability density of the Gaussian distribution centered at $\tilde x$ with covariance matrix $\sigma_t^2I$, $\mu_\mathrm{data}$ is the data distribution, $\lambda$ is a positive weighting function, and  $s_\theta$ is a neural network that takes $x$ as its input. Regarding the time interval $[\varepsilon,T]$, $\varepsilon>0$ is selected to be as small as possible without causing numerical instability, and $T<\infty$ is chosen to be sufficiently large so that the approximation of the probability density $\int p_{\sigma_T}(\cdot|\tilde x)d\mu(\tilde x)$ by   a Gaussian distribution centered at zero is valid. \cite{SE20} notes that setting $\alpha=1$  yields empirically satisfactory results.

 Theoretically, sampling methods such as annealed Langevin dynamics \cite{SE19}, and probability flow ODE \cite{SSKKEP} are assured to sample from the intended distribution if discretization errors are negligible and the correct score function is given. However, the training objective \eqref{generalloss} does not yield the score function, even if the neural network $s_\theta$ has sufficient capacity. Nevertheless, a model in \cite{SE20} was trained using a time-discrete version of \eqref{generalloss} and produced satisfactory samples via annealed Langevin dynamics. To understand this phenomenon, a first step would be to analyze the dynamics resulting from replacing the score function in the Langevin equation with the minimizing function of \eqref{generalloss}, assuming sufficient model capacity. We describe this sampling process in Section \ref{sec-3}. Additionally, a natural question arises about applying the same analysis to the probability flow ODE. The sampling process corresponding to the probability flow ODE will also be described in Section \ref{sec-3}. 

In Section \ref{sec-4}, we explain the efficacy of the two sampling processes. We first show that both samplers are noise-removed, time-scaled version of the autonomous system
\[
\dot x=s_\theta(x).
\]
Then we study the asymptotic behavior of this system to argue that the model generates satisfactory samples if the $s_\theta$ is close to the optimal model derived in Section \ref{sec-2}, and that overfitting can occur if the training loss is too small. 


\section{Training the model}\label{sec-2}
In this section, we introduce the discrete-time training objective from \cite{SE19, SE20} and its continuous-time analogue from \cite{SSKKEP}. We then compare the optimal models derived from a noise conditional network and an unconditional one, both assuming sufficient model capacity.

\subsection{Training objective}\label{sec2-1}
Consider a data distribution $\mu_\mathrm{data}$ and a sequence of noise scales $0<\sigma_1<\sigma_2<\dots<\sigma_T$. By adding  Gaussian noise corresponding to each $\sigma_t$ to $\mu_{\mathrm{data}}$, we derive the probability densities $p_{\sigma_t}$, $t=1,\dots,T$. Specifically, the parametrized family of functions $\{p_r\}_{r>0}$ is defined as follows:
\begin{equation}\label{ps}
p_{r}(x):= \int_{\mathbb R^d} p_{r}(x|\tilde x)d\mu_\mathrm{data}(\tilde x),\quad p_{r}(x|\tilde x):=\frac{1}{(2\pi r^2)^{\frac{d}{2}}}e^{-\frac{|x-\tilde x|^2}{2 r^2}}. 
\end{equation}
The training objective in \cite{SE19, SE20} has the form
\begin{equation}\label{discreteloss}
\underset{\theta}{\operatorname{minimize}}\sum_{t=1}^T \lambda(t)\int_{\mathbb R^d}\int_{\mathbb R^d}\bigg[
\big|s_\theta( \sigma_t, x)-\nabla_x \log   p_{\sigma_t}(  x|\tilde x)\big|^2  
 p_{\sigma_t}(  x|\tilde x)\bigg]d xd\mu_\mathrm{data}(\tilde x).
\end{equation}
Assuming sufficient model capacity, the minimizing function is equal to the noise-dependent score function \cite{Vincent}, i.e., $s_{\theta^*}(\sigma_t,x)\equiv \nabla \log p_{\sigma_t}(x)$ $(t=1,\dots,T)$ for the optimal $\theta^*$. \cite{SSKKEP} interpreted this objective as a discrete-time approximation of the following continuous-time analogue: consider time-varying noise $\{\sigma_t\}_{\varepsilon\leq t\leq T}$ with magnitude increasing over time, and the objective of the form
 \begin{equation}\label{VEloss}
\underset{\theta}{\operatorname{minimize}}\int_\varepsilon^T\int_{\mathbb R^d}\int_{\mathbb R^d} \big|s_\theta( \sigma_t, x)-\nabla_x \log   p_{\sigma_t}(  x|\tilde x)\big|^2  p_{\sigma_t}(  x|\tilde x)d xd\mu_\mathrm{data}(\tilde x)\lambda(t)dt.
\end{equation}
Analogously as above, we have $s_{\theta^*}(\sigma_t,x)\equiv \nabla \log p_{\sigma_t}(x)$ for $\varepsilon\leq t\leq T$, where $\theta^*$ denotes the optimal parameters. We note that $p_{\sigma_t}(\cdot|\tilde x)$ may be interpreted as the density of the solution to the SDE
\[
\begin{cases}
\displaystyle dX_t=g(t)dW_t, \quad t\geq0,\\
\displaystyle X_0=\tilde x,
\end{cases}
\]
where the following relationship exists between $\sigma_t$ and $g(t)$ \cite{SSKKEP}:
\begin{equation}\label{sigmag}
\sigma_t=\sqrt{\int_0^tg(\eta)^2d\eta},\quad t\geq0.
\end{equation}

In \cite{SE20}, $s_\theta(x,\sigma_t)$ in \eqref{discreteloss} was replaced by $s_\theta(x)/\sigma_t$. In this case, even if we assume that the model $s_\theta$ has an ideal architecture with infinite model capacity, the minimized model function $s_{\theta^*}(x)/\sigma_t$ does not match the score function, as will be demonstrated in the following subsection.

\subsection{Optimal model}
We replace  $s_\theta(x,\sigma_t)$ in \eqref{VEloss} by $\sigma_t^{-\alpha}s_\theta(x)$ and consider the following objective:
\[
\underset{\theta}{\operatorname{minimize}}\int_\varepsilon^T\int_{\mathbb R^d}\int_{\mathbb R^d} \big|\sigma_t^{-\alpha}s_\theta( x)-\nabla \log   p_{\sigma_t}(  x|\tilde x)\big|^2  p_{\sigma_t}(  x|\tilde x)d xd\mu_\mathrm{data}(\tilde x)\lambda(t)dt.
\]
Note that above expression integrates denoising score matching objectives over time, which is equivalent to minimizing explicit score matching objectives \cite{Vincent}. In other words, we have the following equivalent objective:
\[
\underset{\theta}{\operatorname{minimize}}\int_\varepsilon^T \int_{\mathbb R^d} \big|\sigma_t^{-\alpha}s_\theta( x)-\nabla \log   p_{\sigma_t}(  x)\big|^2  p_{\sigma_t}(  x)d x \lambda(t)dt.
\]
By the Fubini theorem, the objective above is equivalent to minimizing the following with respect to $\theta$:
\begin{align}\label{fubini}
\begin{aligned}
&\int_{\mathbb R^d} \int_\varepsilon^T \big|\sigma_t^{-\alpha}s_\theta(  x)-\nabla \log   p_{\sigma_t}(  x)\big|^2  p_{\sigma_t}(  x)\lambda(t)dt d  x\\
=&\int_{\mathbb R^d} \Bigg[\mathcal I_1|s_\theta(  x)|^2  -2 \bigg\langle s_\theta(  x),\mathcal I_2\bigg\rangle  +  \mathcal I_3\Bigg] d  x
=\int_{\mathbb R^d}\Bigg[\mathcal I_1  \Bigg|s_\theta(  x)-\frac{\mathcal I_2}{\mathcal I_1}\Bigg|^2-\frac{\big|\mathcal I_2\big|^2}{\mathcal I_1} +  \mathcal I_3\Bigg] d  x,
\end{aligned}
\end{align}
under the assumption that the following three integrals have finite values for a.e. $  x\in\mathbb R^d$:
\begin{align}\label{3int}
\begin{aligned}
&\mathcal I_1:=\int_\varepsilon^T \sigma_t^{-2\alpha}\lambda(t)  p_{\sigma_t}(  x)dt,\\
&\mathcal I_2:=\int_\varepsilon^T \sigma_t^{-\alpha} \lambda(t)\nabla  p_{\sigma_t}(  x)dt, \\
&\mathcal I_3:=\int_\varepsilon^T \lambda(t)\frac{|\nabla  p_{\sigma_t}(  x)|^2}{  p_{\sigma_t}(  x)}dt.
\end{aligned}
\end{align}

Assuming sufficient model capacity, our optimal $s_\theta$ would be equal to the following for a.e. $x\in\mathbb R^d$:
\[
s_*(x):=\frac{\mathcal I_2}{\mathcal I_1}=\frac{\int_\varepsilon^T \sigma_t^{-\alpha} \lambda(t)\nabla  p_{\sigma_t}(  x )dt}{\int_\varepsilon^T \sigma_t^{-2\alpha}\lambda(t)  p_{\sigma_t}(  x )dt}.
\]

From now on, we restrict our attention to the special case  where  $\lambda(t)= \big(g(t)\big)^2$, for $\sigma_t$ and $g(t)$ satisfying the relationship \eqref{sigmag}. This choice of $\lambda$, called {\it likelihood weighting}, ensures that the objective upper bounds the negative log-likelihood of the model \cite{SDME}.  The case with $\alpha=1$ and $g(t)$ being an exponential function was addressed in \cite{SE20}, within a discrete-time formulation.

For $0<a<b<\infty$, $s\in\mathbb R$, $z\in\mathbb R$, we introduce the following notation:
\begin{equation}\label{phi}
\Phi_a^b(s,z):=\int_{\frac{1}{b^2}}^\frac{1}{a^2}u^{s-1}e^{-zu}du.
\end{equation}
Note that for $s>0$, $z>0$, we may express $\Phi_a^b$ in terms of incomplete gamma functions as follows:
\[
\Phi_a^b(s,z)= \frac{1}{z^s}\int_{\frac{z}{b^2}}^{\frac{z}{a^2}} t^{s-1}e^{-t}dt=\frac{\Gamma\big(s,\frac{z}{b^2}\big)-\Gamma\big(s,\frac{z}{a^2}\big)}{z^s}=\frac{\gamma\big(s,\frac{z}{a^2}\big)-\gamma\big(s,\frac{z}{b^2}\big)}{z^s},
\]
where $\Gamma$ and $\gamma$ are the upper and the lower incomplete gamma functions, i.e.,
\[
\Gamma(s,z):=\int_z^\infty t^{s-1}e^{-t}dt, \quad \gamma(s,z):=\int_0^z t^{s-1}e^{-t}dt.
\] 
By utilizing the above notation, we can conveniently express $s_*$, as demonstrated below:
\begin{proposition}\label{T1}
Assume that $\mu_\mathrm{data}$ is a probability measure on $\mathbb R^d$, $g:[0,\infty)\to\mathbb R$ is a positive bounded measurable function, $\alpha\in\mathbb R$, and $0<\varepsilon<T<\infty$. Define $p_r$, $\sigma_t$, and $\Phi_a^b(s,z)$ as in \eqref{ps}, \eqref{sigmag}, \eqref{phi}, respectively. Then for any measurable $s_\theta:\mathbb R^d\to\mathbb R^d$, we have
\begin{align*}
\begin{aligned}
&\int_\varepsilon^T \int_{\mathbb R^d} \left| \sigma_t^{-\alpha}s_\theta( x)-\nabla \log   p_{\sigma_t}(  x)\right|^2  p_{\sigma_t}(  x)d x (g(t))^2dt\\
&=\int_{\mathbb R^d}\bigg[ 
  \frac{1}{( 2\pi )^{\frac{d}{2}}} \int_{\mathbb R^d}\Phi_{\sigma_\varepsilon}^{\sigma_T}\left(\frac{d+2\alpha-2}{2},\frac{|x-\tilde x|^2}{2}\right) d\mu_\mathrm{data}(\tilde x)    \bigg] 
 |s_\theta(  x)-s_*(x)|^2dx+C
\end{aligned}
\end{align*}
for some   constant $0<C<\infty$, where
\begin{equation}\label{sstar}
s_*(x):=\frac{\int_{\mathbb R^d}\Phi_{\sigma_\varepsilon}^{\sigma_T}\Big(\frac{d+\alpha}{2},\frac{|x-\tilde x|^2}{2}\Big)(\tilde x-x)d\mu_\mathrm{data}(\tilde x)}{\int_{\mathbb R^d}\Phi_{\sigma_\varepsilon}^{\sigma_T}\Big(\frac{d+2\alpha-2}{2},\frac{|x-\tilde x|^2}{2}\Big)d\mu_\mathrm{data}(\tilde x)}.
\end{equation}
\end{proposition}
\begin{proof}
The proof involves explicitly calculating the first two integrals in \eqref{3int}, deriving an appropriate upper bound for the third integral in \eqref{3int}, and then substituting these results into \eqref{fubini}. For detailed steps, see Appendix \ref{T1proof}.
\end{proof}
From Proposition \ref{T1}, we can see that the optimal model with multiplicative noise conditioning is equal to $\sigma_t^{-1}s_*(x)$, as opposed to the ideal model $s_{\theta^*}(\sigma_t,x)\equiv \nabla \log p_{\sigma_t}(x)$ in Subsection \ref{sec2-1}.

\section{Derivation of sampling methods}\label{sec-3}
In this section, we review the theoretical explanation of two sampling methods, annealed Langevin dynamics \cite{SE19, SE20} and probability flow ODE \cite{SSKKEP}, and argue that those explanations cannot be applied to the model with multiplicative noise conditioning.

\subsection{Annealed Langevin dynamics}
%
%
%
%
%
%
%
%

In annealed Langevin dynamics \cite{SE19}, a sample is first drawn from a well-known distribution, such as a uniform or Gaussian distribution. Then, Langevin dynamics algorithms are applied sequentially according to a predefined sequence of noise scales, $\sigma_{t_n} > \sigma_{t_{n-1}} > \dots > \sigma_{t_1}$, in the order of decreasing noise magnitude. A Langevin dynamics algorithm corresponding to $\sigma_{t_k}$ entails numerically solving the following SDE, up to sufficiently large time:
\begin{equation}\label{eachlangevin}
d\tilde X_t=s_\theta(\sigma_{t_k},\tilde X_t)dt+\sqrt{2}d\tilde W_t,\quad t\geq0.
\end{equation}
The approximation $s_\theta(\sigma_{t_k},x)\approx \nabla \log p_{\sigma_{t_k}}(x)$ holds under the condition of sufficient model capacity. Upon solving \eqref{eachlangevin}, we  obtain a sample from the distribution $p_{\sigma_{t_k}}$.  By sequentially solving \eqref{eachlangevin} for $k=n,n-1,\dots,1$, the distribution of the sample gradually approaches $p_{\sigma_{t_1}}$, which closely approximates $\mu_\mathrm{data}$ if $\sigma_{t_1}\ll1$. 

However, this is not the case for the model with multiplicative noise conditioning. We replace $s_\theta(\sigma_{t_k},x)$ by $\sigma_{t_k}^{-\alpha}s_\theta(x)$ to get the following SDE:
\begin{equation}\label{XSDE1}
d\tilde X_t=\sigma_{t_k}^{-\alpha} s_\theta(  \tilde X_t)dt+\sqrt{2}d\tilde W_t,\quad t\geq0. 
\end{equation}
 If the model capacity is sufficient, then we would have $\sigma_{t_k}^{-\alpha}s_\theta(x)\approx \sigma_{t_k}^{-\alpha}s_*(x)\neq \nabla \log p_{\sigma_{t_k}}(x)$, where $s_*$ is defined as in \eqref{sstar}. Therefore, we cannot use the well-known asymptotic behaviors of the Langevin dynamics algorithm to understand the behavior of \eqref{XSDE1}. 

\subsection{Probability flow ODE}\label{sec-3-2}
In this subsection, we review the probability flow ODE \cite{SSKKEP} and apply this method to the model with multiplicative noise conditioning. Note that $p_{\sigma_t}$ is the probability density of the solution $X_t$ to the following SDE at time $t>0$:
\begin{equation}\label{VESDE}
dX_t=g(t)dW_t,\quad X_0\sim \mu_\mathrm{data}.
\end{equation}
Hence $p_{\sigma_t}$ satisfies its associated Fokker-Planck equation:
\[
\partial_tp_{\sigma_t}(x)=\frac{1}{2}\nabla^2 \left[p_{\sigma_t}(x)g(t)^2\right].
\]
As in \cite{SSKKEP}, we may use the identity
\[
\nabla \left[p_{\sigma_t}(x)g(t)^2 \right]= p_{\sigma_t}(x)g(t)^2 \nabla\log p_{\sigma_t}(x) 
\]
to rewrite this as
\[
\partial_t p_{\sigma_t}(x)= \frac{1}{2} \nabla\cdot\bigg[  p_{\sigma_t}(x)g(t)^2 \nabla\log  p_{\sigma_t}(x)    \bigg].
\]
Define $  \rho_t:=p_{\sigma_{T-t}}$. Then $ \rho_t$ satisfies
\[
\partial_t \rho_t(x)=  -\frac{ 1}{2}\nabla\cdot\bigg[   \rho_t(x)g(T-t)^2 \nabla\log p_{\sigma_{T-t}}(x)    \bigg].
\]
Note that this is the Fokker-Planck equation associated to the following ODE, with random initial data:
\[
\begin{cases}
\displaystyle d \bar X_t=  \frac{ g(T-t)^2}{2}   \nabla\log p_{\sigma_{T-t}}(\bar X_t) dt,  \\
\displaystyle \bar X_0\sim p_{\sigma_T}.
\end{cases}
\]
The probability density of $\bar X_t$ is equal to $p_{\sigma_{T-t}}$, for $t\in[0,T)$. Hence, if we replace the initial distribution with a well-known distribution that is sufficiently close to $p_{\sigma_T}$, $\nabla\log p_{\sigma_{T-t}}(\bar X_t)$ with $s_\theta(\sigma_{T-t},\bar X_t)$, and assume sufficient model capacity,  then the solution to the resulting ODE with random initial data will reach to a sample from a distribution that is close to $p_{\sigma_\varepsilon}$ at time $t=T-\varepsilon$.  

If we consider the model with multiplicative noise conditioning, then the sampling process would be
\begin{equation}\label{XSDE2}
d \bar X_t=  \frac{ g(T-t)^2}{2}    \sigma_{T-t}^{-\alpha}s_\theta(\bar X_t) dt,~ 0\leq t\leq T-\varepsilon.
\end{equation}
 If the model capacity was sufficient, we would have $\sigma_{T-t}^{-\alpha}s_\theta(x)\approx \sigma_{T-t}^{-\alpha}s_*(x)\neq \nabla \log p_{\sigma_{T-t}}(x)$, where $s_*$ is defined as in \eqref{sstar}. Therefore, the behavior of \eqref{XSDE2} cannot be understood using the previous argument involving the Fokker-Planck equations. 
%
%
%
%
%
%
%
%
%

\section{Behavior of sampling processes}\label{sec-4}
In this section, we discuss effectiveness of sampling processes for the model with multiplicative noise conditioning introduced in \cite{SE20}. For this, we first show that we can obtain the following autonomous system from both \eqref{XSDE1} and \eqref{XSDE2}, after removal of noise and appropriate time scaling:
\begin{equation}\label{dynamicalsystem}
\frac{dx}{dt}=s_\theta(x),\quad t\in\mathbb R.
\end{equation}
Then we study the asymptotic behavior of this ODE, which will give us some insight about efficacy of the sampling processes \eqref{XSDE1} and \eqref{XSDE2}. 

\subsection{Derivation of the autonomous system}
In this subsection, we derive the autonomous system \eqref{dynamicalsystem} from \eqref{XSDE1} and \eqref{XSDE2}, respectively.
\subsubsection{Derivation from annealed Langevin dynamics}

As for $\tilde X_t$ in \eqref{XSDE1}, we define $\tilde Y_t:=\tilde X_{\sigma_{t_k}^\alpha t}$. Then we have
\begin{align*}
\begin{aligned}
\tilde Y_t-\tilde Y_s&=\tilde X_{\sigma_{t_k}^\alpha t}-\tilde X_{\sigma_{t_k}^\alpha s}=  \int_{\sigma_{t_k}^\alpha s}^{\sigma_{t_k}^\alpha t} \sigma_{t_k}^{-\alpha} s_\theta(\tilde  X_\eta)d\eta  + \sqrt{2}(\tilde W_{\sigma_{t_k}^\alpha t}-\tilde W_{\sigma_{t_k}^\alpha s})\\
&= \int_{s}^{t} s_\theta(\tilde Y_{\eta})d\eta  + \sqrt{2}\sigma_{t_k}^{\frac{\alpha}{2}}(\tilde W_t'-\tilde W_s'),
\end{aligned}
\end{align*}
for the Brownian motion
\[
\tilde W_t':=\sigma_{t_k}^{-\frac{\alpha}{2}} \tilde W_{\sigma_{t_k}^{\alpha}t}.
\]
Hence $\tilde Y_t$ satisfies the SDE
\[
d  \tilde Y_t=s_\theta( \tilde Y_t)dt+\sqrt{2}\sigma_{t_k}^{\frac{\alpha}{2}}d\tilde W_t',\quad t\geq0.
\]
By removing noise from this SDE, we get \eqref{dynamicalsystem}. Hence studying the system \eqref{dynamicalsystem} can be the first step toward understanding the well-functioning of the sampler \eqref{XSDE1}.

\subsubsection{Derivation from probability flow ODE}

As for $\bar X_t$ in \eqref{XSDE2}, we define  $\bar Y_t:=\bar X_{ u^{-1}(t)}$, where
\[
\displaystyle u(t):=  \frac{ 1}{2}\int_{\sigma_{T-t}^2}^{\sigma_T^2} \eta^{-\frac{\alpha}{2}}d\eta=\begin{cases}\displaystyle  \frac{\sigma_T^{2-\alpha}-\sigma_{T-t}^{2-\alpha}}{2-\alpha}\quad&\mbox{if}\quad \alpha\neq2,\\ \displaystyle  \log \sigma_T-\log \sigma_{T-t} \quad&\mbox{if} \quad \alpha=2.\end{cases}
\]
Then
\[
\bar Y_t-\bar Y_s=\bar X_{ u^{-1}(t)}-\bar X_{ u^{-1}(s)}
= \frac{ 1}{2}\int_{u^{-1}(s)}^{u^{-1}(t)}(g(T-\eta))^2\sigma_{T-\eta}^{-\alpha} s_\theta(\bar X_\eta)d\eta
= \int_{s}^{t} s_\theta(\bar Y_{\eta})d\eta.
\]

Hence $\bar Y_t$ satisfies the following ODE:
\[
d\bar Y_t=s_\theta(\bar Y_t)dt,\quad 0\leq t\leq u(T-\varepsilon).
\] 
Note that for $0<\sigma_\varepsilon\ll1$ and $\sigma_T\gg1$, we have $u(T-\varepsilon)\gg1$. Hence studying asymptotic behavior of solutions to \eqref{dynamicalsystem} as $t\to\infty$ would provide us some insight to the efficacy of the sampling process \eqref{XSDE2}.

\subsection{Analysis of the autonomous system}
In this subsection, we will study the long time behavior of a solution $x(t)$ to the system \eqref{dynamicalsystem} by analyzing the dynamics of the function $ L_{\sigma_\varepsilon}^{\sigma_T}(\frac{d+\alpha}{2}-1,x(t))$, where $L_a^b(s,x)$ for $0<a<b<\infty,~s\in\mathbb R,~x\in\mathbb R^d$ is defined as
\[
L_a^b(s,x):=\int_{\mathbb R^d}\Phi_{a}^{b}\Big(s,\frac{|x-\tilde x|^2}{2}\Big)d\mu_\mathrm{data}(\tilde x).
\]
Note that studying dynamics of $ L_{\sigma_\varepsilon}^{\sigma_T}(\frac{d+\alpha}{2}-1,x(t))$ is equivalent to studying that of $ P_{\sigma_\varepsilon}^{\sigma_T}(\frac{d+\alpha}{2}-1,x(t))$, defined as
\[
P_a^b(s,x):=\frac{\int_{\mathbb R^d}\Phi_{a}^{b}\Big(s,\frac{|x-\tilde x|^2}{2}\Big)d\mu_\mathrm{data}(\tilde x)}{C_a^b(s)},
\]
where $C_a^b(s)$ is the normalizing constant which ensures
\[
\int_{\mathbb R^d}P_a^b(s,x)dx=1.
\]
Note that the existence of such constant is guaranteed by the integrability of $L_a^b(s,\cdot)$ (Lemma \ref{Lemmaconstant}).
$P_a^b(s,\cdot)$ can be seen as the probability density function obtained by the convolution of the two probability measures $\mu_\mathrm{data}$ and $\frac{1}{C_a^b(s)}\Phi_a^b\big(s,\frac{|x|^2}{2}\big)dx$. Since $\frac{1}{C_a^b(s)}\Phi_a^b\big(s,\frac{|x|^2}{2}\big)dx$ is a smooth radially symmetric function that decreases in $|x|$ (Lemma \ref{Lemmadecrease}) and vanishes at infinity (Lemma \ref{Lemmaasymptotic}),  its convolution with $\mu_\mathrm{data}$ can be seen as a smoothed version of $\mu_\mathrm{data}$. Therefore, if we can show that the flows of \eqref{dynamicalsystem} move to regions with high value of $L_{\sigma_\varepsilon}^{\sigma_T}(\frac{d+\alpha}{2}-1,\cdot)$, or equivalently, 
\begin{equation}\label{pdata}
\tilde p_\mathrm{data}(\cdot):=P_{\sigma_\varepsilon}^{\sigma_T}\left(\frac{d+\alpha}{2}-1,\cdot\right),
\end{equation}
then we would be able to infer that the flows move to the modes of $\mu_\mathrm{data}$, which is the desirable distribution that we want to generate samples from.

From now on, we study the well-posedness and long-time behavior of \eqref{dynamicalsystem} using the new notation $e(x):=s_\theta(x)-s_*(x)$. With this,  \eqref{dynamicalsystem} becomes
\begin{equation}\label{mainsystem}
 \frac{dx}{dt}=s_*(x)+e(x),\quad t\in \mathbb R.
\end{equation}
By Lemma \ref{Lemmacinfty}, we observe that $s_*$ is of class $C^\infty$ for any $0 < \sigma_\varepsilon < \sigma_T < \infty$ and $\alpha \in \mathbb{R}$. Hence, assuming that $e: \mathbb{R}^d \to \mathbb{R}^d$ is locally Lipschitz continuous, we obtain local-in-time existence and uniqueness of a solution to \eqref{mainsystem} for any given initial condition, by the standard Cauchy-Lipschitz theory. In the sequel, we will see that the global-in-time existence will be established if we additionally assume compact support of $\mu_\mathrm{data}$ and that $|e(x)|$ grows at most linearly as $|x|\to\infty$.
\begin{assumption}\label{Assumptione1} $0 < \sigma_\varepsilon < \sigma_T < \infty$, $\alpha \in \mathbb{R}$, $\mu_\mathrm{data}$ is compactly suppported, $e: \mathbb{R}^d \to \mathbb{R}^d$ is locally Lipschitz continuous and satisfies the following linear growth condition:
\[
\sup_{x\in\mathbb R^d}\frac{|e(x)|}{1+|x|}<\infty.
\]
\end{assumption}

The following proposition is a key ingredient in showing the global well-posedness of \eqref{mainsystem}.
\begin{proposition}\label{Theoremlimsup}
For $0<\sigma_\varepsilon<\sigma_T<\infty$, $\alpha\in\mathbb R$, and a compactly supported probability measure $\mu_\mathrm{data}$ on $\mathbb R^d$, $s_*$ defined in \eqref{sstar} satisfies
\[
\lim_{R\to\infty}\sup_{|x|\geq R}\left|\frac{|s_*(x)|}{|x|}-\sigma_T^{\alpha-2}\right|=0.
\]
\end{proposition}
\begin{proof}
We decompose $\frac{|s_*(x)|}{|x|}=\mathcal I_1\times\mathcal I_2$, where
\[
\mathcal I_1:=\frac{ \int_{\mathbb R^d}\Phi_{\sigma_\varepsilon}^{\sigma_T}\left(\frac{d+\alpha }{2},\frac{|x-\tilde x|^2}{2}\right)d\mu_\mathrm{data}(\tilde x) }{\int_{\mathbb R^d}\Phi_{\sigma_\varepsilon}^{\sigma_T}\left(\frac{d+2\alpha-2}{2},\frac{|x-\tilde x|^2}{2}\right)d\mu_\mathrm{data}(\tilde x)},
\]
and
\[
\mathcal I_2:=\frac{\left|\int_{\mathbb R^d}\Phi_{\sigma_\varepsilon}^{\sigma_T}\left(\frac{d+\alpha }{2},\frac{|x-\tilde x|^2}{2}\right)(\tilde x-x)d\mu_\mathrm{data}(\tilde x)\right|}{\int_{\mathbb R^d}\Phi_{\sigma_\varepsilon}^{\sigma_T}\left(\frac{d+\alpha}{2},\frac{|x-\tilde x|^2}{2}\right)|x|d\mu_\mathrm{data}(\tilde x)}.
\]
Consider the simple case where $\mu_\mathrm{data}$ is a Dirac measure at some point $\tilde x\in\mathbb R^d$. By using
\[
\lim_{z\to\infty} ze^{\frac{z}{b^2}}\Phi_{a}^{b}(s,z) =\frac{1}{b^{2(s-1)}},
\]
we can show that the uniform limit of $\mathcal I_1$ as $|x|\to\infty$ is equal to $\sigma_T^{\alpha-2}$ (Lemma \ref{Lemmaasymptotic}). And $\mathcal I_2=\frac{|\tilde x-x|}{|x|}$ clearly tends to $1$. Detailed proof for the general case is provided in Appendix \ref{Prooflimsup}.
\end{proof}

Now we are ready to state and prove the global well-posedness theorem.

\begin{theorem}\label{Theoremexistence}
Let $x_0\in\mathbb R^d$ and assume that Assumption \ref{Assumptione1} holds. Then the following initial value problem admits a unique solution $x(t)$ for $t\in(-\infty,\infty)$:
\[
\begin{cases}
\displaystyle \frac{dx}{dt}=s_*(x)+e(x),\\
\displaystyle x(0)=x_0.
\end{cases}
\]
\end{theorem}
\begin{proof}
Local existence and uniqueness is clear from the local Lipschitz continuity of $s_*+e$.  By Proposition \ref{Theoremlimsup}, $|s_*+e|$ grows at most linearly in $|x|$. We combine this and the Gr\"onwall inequality to obtain the global existence of the solution. Detailed proof is provided in Appendix \ref{Proofexistence}.
\end{proof}

To explain the behavior of the solution to \eqref{mainsystem}, we introduce the concepts of attracting sets and fundamental neighborhoods.
\begin{definition}\label{Defattract} \cite{Ruelle}
For a dynamical system $\{\phi_t\}_{t\in\mathbb R}$ on $\mathbb{R}^d$, a closed set $\Lambda \subset \mathbb{R}^d$ is called an \textit{attracting set} if there exists a neighborhood $U$ of $\Lambda$ such that for any neighborhood $V$ of $\Lambda$, we have $\phi_t(U) \subset V$ for sufficiently large $t$, and satisfies
\[
\Lambda=\phi_T(\Lambda)=\bigcap_{t\geq T}\phi_t(U),\quad\forall~T\in\mathbb R.
\]
 In this case, $U$ is called a \textit{fundamental neighborhood} of $\Lambda$. The open set
\[
W:=\bigcup_{t\in\mathbb R}(\phi_t)^{-1}(U)
\]
is called the {\it basin of attraction} of $\Lambda$.
\end{definition}
A more generalized definition and several equivalent conditions can be seen in \cite{Ruelle}. The basin of attraction is the set of all points that are attracted to the attracting set as $t\to\infty$, so it is independent of the choice of a fundamental neighborhood. 
In terms of attracting sets and basins of attractions, the following theorem explains that the flows of \eqref{mainsystem} move toward regions with high values of $L_{\sigma_\varepsilon}^{\sigma_T}\left(\frac{d+\alpha}{2}-1,\cdot\right)$, or equivalently $\tilde p_\mathrm{data}$, the density of a smoothed version of $\mu_\mathrm{data}$. 
\begin{theorem}\label{Theoremattract}
Let $0<a<b<\infty$ be given and assume that Assumption \ref{Assumptione1} holds. Let $\Gamma$ be a path-connected component of the non-empty superlevel set $\{x\in\mathbb R^d:L_{\sigma_\varepsilon}^{\sigma_T}\left(\frac{d+\alpha}{2}-1,x\right)>a\}$. Define
\[
 A:=\left\{x\in \Gamma:L_{\sigma_\varepsilon}^{\sigma_T}\left(\frac{d+\alpha}{2}-1,x\right)\geq b\right\}.
\]
Suppose that $\Gamma\setminus A$
 is disjoint from the set 
\[
E:=\{x\in\mathbb R^d:|s_*(x)|^2\leq -\langle s_*(x),e(x)\rangle\}.
\]
 Then there exists a compact attracting set $\Lambda\subset A$ of \eqref{mainsystem}, whose basin of attraction contains $\Gamma$.
\end{theorem}
\begin{proof}
We only provide motivation here. Denote by $\phi_t$ the flow generated by the system \eqref{mainsystem}. Then we have
\[
\frac{d}{dt}L_{\sigma_\varepsilon}^{\sigma_T}\left(\frac{d+\alpha}{2}-1,\phi_t(x)\right)>0
\]
as long as $\phi_t(x)\notin E$. Hence, for $x\in \Gamma$, $L(\phi_t(x))$ increases in $t\geq0$ until $\phi_t(x)$ enters $A$. This motivates us to introduce the following candidates for fundamental neighborhoods:
\[
U_c:=\{x\in\Gamma: L(x)>c\},\quad c\in(a,b).
\]
For any $(c_1,c_2)\subset(a,b)$, $\phi_t(U_{c_1})\subset U_{c_2}$ for sufficiently large $t\geq0$. Detailed proof is provided in Appendix \ref{Proofattract}.
\end{proof}

\begin{remark}\label{Remark}
We may write a straightforward corollary of Theorem \ref{Theoremattract} by replacing $E$ by any set $\tilde E\subset\mathbb R^d$ containing $E$. One such example is 
\begin{align*}
\begin{aligned}
\tilde E:&=\left\{x\in\mathbb R^d: |s_*(x)|\leq |e(x)|\right\}\\
&=\bigg\{x\in\mathbb R^d: \left|\nabla_x\log L_{\sigma_\varepsilon}^{\sigma_T}\left(\frac{d+\alpha}{2}-1,x\right)\right|\leq\frac{L_{\sigma_\varepsilon}^{\sigma_T}\left(\frac{d+2\alpha}{2}-1,x\right)}{L_{\sigma_\varepsilon}^{\sigma_T}\left(\frac{d+\alpha}{2}-1,x\right)}|e(x)| \bigg\}.
\end{aligned}
\end{align*}
Note that since
\[
\lim_{R\to\infty}\sup_{|x|\geq R}\left|\frac{L_{\sigma_\varepsilon}^{\sigma_T}\left(\frac{d+2\alpha}{2}-1,x\right)}{L_{\sigma_\varepsilon}^{\sigma_T}\left(\frac{d+\alpha}{2}-1,x\right)}-\sigma_T^{-\alpha}\right|=0,
\] the fraction is uniformly bounded in $x\in\mathbb R^d$ (Lemma \ref{Lemmafraction}).
Hence, we can roughly say that if $|e(\cdot)|$ is small, then the set $\tilde E$ will consist of $x\in\mathbb R^d$ having small values of 
\[
\left|\nabla_x\log L_{\sigma_\varepsilon}^{\sigma_T}\left(\frac{d+\alpha}{2}-1,x\right)\right|=|\nabla_x\log\tilde p_\mathrm{data}(x)|.
\] 
Consider a path-connected component $\Gamma$ of a strict superlevel set of $L_{\sigma_\varepsilon}^{\sigma_T}\left(\frac{d+\alpha}{2}-1,\cdot\right)$, or equivalently, that of $\log\tilde p_\mathrm{data}(\cdot)$. If $\{x\in\Gamma:\log\tilde p_\mathrm{data}(x)<\tilde b\}$ is disjoint from $\tilde E$, then the flows starting from $\Gamma$ would move into $\{x\in\Gamma:\log\tilde p_\mathrm{data}(x)\geq\tilde b\}$. If $|s_\theta-s_*|=|e|$ is small, then $\tilde E$ will consist only of $x\in\mathbb R^d$ with small $|\nabla_x\log\tilde p_\mathrm{data}(x)|$, so that we would be able to choose relatively large $\tilde b$, and the flows starting from $\Gamma$ would pass through the region with large $|\nabla_x\log\tilde p_\mathrm{data}(x)|$, and move into a region with high values of $\log\tilde p_\mathrm{data}$. In summary, we can infer from Theorem \ref{Theoremattract} that if $|e(\cdot)|$ is small, then the flows of \eqref{mainsystem} will move into regions with high density $\tilde p_\mathrm{data}$.
\end{remark}

Next, we  investigate the special case where $e(x)\equiv0$, i.e., $s_\theta(x)\equiv s_*(x)$. In other words, we consider the following assumption, which is a special case of Assumption \ref{Assumptione1}.
\begin{assumption}\label{Assumptione2} $0 < \sigma_\varepsilon < \sigma_T < \infty$, $\alpha \in \mathbb{R}$, $\mu_\mathrm{data}$ is compactly suppported, and $e: \mathbb{R}^d \to \mathbb{R}^d$ is identically equal to zero.
\end{assumption}
On one hand, we have
\[
E=\{x\in\mathbb R^d:\nabla_x\log\tilde p_\mathrm{data}(x)=0\}
\]
in Theorem \ref{Theoremattract}, so  that the flows starting from $\Gamma$ will pass through the regions with nonzero $\nabla_x\log\tilde p_\mathrm{data}$, and move into regions that have critical points of $\log\tilde p_\mathrm{data}$ on their boundaries. On the other hand, we have the following theorem which simply states that all flows of \eqref{mainsystem} converge to the set of critical points of  $\log\tilde p_\mathrm{data}$.


\begin{theorem}\label{Theoremlasalle}
Assume that Assumption \ref{Assumptione2} holds and let $\phi_t$ be the flow governed by the autonomous system \eqref{mainsystem}. Define
\[
E=\left\{x\in\mathbb R^d: \nabla_x L_{\sigma_\varepsilon}^{\sigma_T}\left(\frac{d+\alpha}{2}-1,x\right)=0\right\}.
\]
Then $\phi_t(x)\to E$ as $t\to\infty$ for all $x\in \mathbb R^d$.
\end{theorem}
\begin{proof}
The following observation motivates us to use LaSalle's invariance principle \cite{LaSalle}, introduced in Proposition \ref{Proplasalle}:
\[
\frac{d}{dt}L_{\sigma_\varepsilon}^{\sigma_T}\left(\frac{d+\alpha}{2}-1,\phi_t(x)\right)=\frac{|\nabla L_{\sigma_\varepsilon}^{\sigma_T}\left(\frac{d+\alpha}{2}-1,\phi_t(x)\right)|^2}{L_{\sigma_\varepsilon}^{\sigma_T}\left(\frac{d+2\alpha}{2}-1,\phi_t(x)\right)}\geq 0,\quad x\in\mathbb R^d.
\]
Detailed proof is provided in Appendix \ref{Prooflasalle}.
\end{proof}
By Theorem \ref{Theoremlasalle}, we can expect that all flows of \eqref{mainsystem} will converge to modes of $\tilde p_\mathrm{data}$, the density of a smoothed version of $\mu_\mathrm{data}$, as $t\to\infty$. Since the goal of the model $s_\theta$ is to generate samples from the desirable distribution $\mu_\mathrm{data}$, or more precisely, from regions with high density with respect to $\mu_\mathrm{data}$, we can infer from Theorem \ref{Theoremlasalle} that this goal is achieved.

Lastly, we theoretically argue that the model will be overfit to the training data if the training loss is too small. Let $\mu_\mathrm{train}$ be the discrete uniform distribution on a finite set
\[
\mathcal X=\{x_1,\dots,x_N\}\subset \mathbb R^d.
\]
Consider the extreme case where the training loss is zero. Then the resulting model $s_*^\mathrm{train}$ would be equal to the modification of $s_*$ by replacing $\mu_\mathrm{data}$ by $\mu_\mathrm{train}$ in \eqref{sstar}, i.e.
\[
s_*^\mathrm{train}(x):=\frac{\frac{1}{N}\sum_{j=1}^N\Phi_{\sigma_\varepsilon}^{\sigma_T}\Big(\frac{d+\alpha}{2},\frac{|x-x_j|^2}{2}\Big)(x_j-x) }{\frac{1}{N}\sum_{j=1}^N\Phi_{\sigma_\varepsilon}^{\sigma_T}\Big(\frac{d+2\alpha-2}{2},\frac{|x-x_j|^2}{2}\Big) }.
\] 
To argue that overfitting occurs, we have to show that the flows of the following system are likely to converge to points that are too close to  $x_i's$, as $t\to\infty$:
\begin{equation}\label{overfitsystem}
\frac{dx}{dt}=s_*^\mathrm{train}(x),\quad t\in\mathbb R.
\end{equation}
The following theorem states that this is the case when $\sigma_\varepsilon$ is sufficiently small.

\begin{theorem}\label{Theoremoverfit}
Let a constant $0<k<\sqrt{\frac{2(d+\alpha+2)}{3(d+\alpha)}}$ be given and assume that $0 < \sigma_\varepsilon < \sigma_T < \infty$, $\alpha>-d$. Then for each $1\leq i\leq N$, there exists a constant $\delta>0$ depending on $d, \alpha, \sigma_T, k, \mathcal X$ such that if $0<\sigma_\varepsilon<\delta$, then the system \eqref{overfitsystem} has an asymptotically stable equilibrium point in the open ball $\{x\in\mathbb R^d:|x-x_i|<k\sigma_\varepsilon\}$, which is a unique equilibrium point in that region.
\end{theorem}
\begin{proof}
The idea is to show that $L_{\sigma_\varepsilon}^{\sigma_T}\left(\frac{d+\alpha}{2}-1,\cdot\right)$ is strictly concave in the open ball, and that 
\[
L_{\sigma_\varepsilon}^{\sigma_T}\left(\frac{d+\alpha}{2}-1,x_i\right)> L_{\sigma_\varepsilon}^{\sigma_T}\left(\frac{d+\alpha}{2}-1,x\right)
\]
 on the boundary $|x-x_i|=k\sigma_\varepsilon$. Then $L_{\sigma_\varepsilon}^{\sigma_T}\left(\frac{d+\alpha}{2}-1,\cdot\right)$ has a unique local maximizer $\hat x_i$ in the open ball  $|x-x_i|<k\sigma_\varepsilon$, which is also a unique stationary point in that region. This motivates the Liapunov function $V(x):=L_{\sigma_\varepsilon}^{\sigma_T}(\hat x_i)-L_{\sigma_\varepsilon}^{\sigma_T}(x)$, finishing the proof. 
Detailed proof is provided in Appendix \ref{Proofoverfit}.
\end{proof}

\section{Conclusion}
In this paper, we have developed a theoretical framework to explain the effectiveness of neural networks trained with multiplicative noise conditioning for sampling processes. By deriving the optimal model that minimizes the training objective proposed in \cite{SE20}, we found that although this objective does not explicitly yield the score function, the trained model can still produce high-quality samples, as empirically demonstrated in \cite{SE20}. We showed that both annealed Langevin dynamics and the probability flow ODE correspond to noise-removed, time-scaled versions of an autonomous system, with its dynamics governed by the trained neural network. Our analysis revealed that the model's performance depends on its proximity to the derived optimal model, with sufficiently accurate networks producing high-quality samples, while excessive training can lead to overfitting.

\appendix

\section{Proofs}
\subsection{Auxiliary lemmas concerning $\Phi_{a}^{b}$}
In this subsection we state and prove some basic properties of $\Phi_{a}^{b}$ that is used throughout the paper. Some of these properties were motivated by those of incomplete gamma functions, which may be found in \cite{Paris} and the references therein. Recall that 
for $0<a<b<\infty$, $s\in\mathbb R$, $z\in\mathbb R$, we define
\[
\Phi_{a}^{b}(s,z):=\int_{\frac{1}{b^2}}^\frac{1}{a^2}u^{s-1}e^{-zu}du. 
\]


\begin{lemma}\label{Lemmaasymptotic}
For any $0<a<b<\infty$, $s\in\mathbb R$, we have
\[
\lim_{z\to\infty} ze^{\frac{z}{b^2}}\Phi_{a}^{b}(s,z) =\frac{1}{b^{2(s-1)}}.
\]
\end{lemma}
\begin{proof}
Recall that for $z>0$, we may express $\Phi_{a}^{b}$ in the following alternative form, by change of variables:
\[
\Phi_{a}^{b}(s,z)= \frac{1}{z^s}\int_{\frac{z}{b^2}}^{\frac{z}{a^2}} t^{s-1}e^{-t}dt,\quad z>0.
\]
Employing L'Hospital's rule yields the following:
\begin{align*}
\begin{aligned}
\lim_{z\to\infty} ze^{\frac{z}{b^2}}\Phi_{a}^{b}(s,z) &=\lim_{z\to\infty} \frac{\int_{\frac{z}{b^2}}^{\frac{z}{a^2}} t^{s-1}e^{-t}dt}{z^{s-1}e^{-\frac{z}{b^2}}} \\
&=\lim_{z\to\infty}\frac{\frac{1}{a^2}(\frac{z}{a^2})^{s-1}e^{-\frac{z}{a^2}}-\frac{1}{b^2}(\frac{z}{b^2})^{s-1}e^{-\frac{z}{b^2}}}{((s-1)z^{s-2}-\frac{1}{b^2}z^{s-1})e^{-\frac{z}{b^2}} }\\
&=\lim_{z\to\infty}\frac{ (\frac{1}{a^2})^{s}e^{\frac{z}{b^2}-\frac{z}{a^2}}- (\frac{1}{b^2})^{s} }{(\frac{s-1}{z}-\frac{1}{b^2} )  }\\
&=\frac{1}{b^{2(s-1)}}.
\end{aligned}
\end{align*}
\end{proof}
\begin{lemma}\label{Lemmaconstant}
For $0<a<b<\infty$, $s\in\mathbb R$, $k\geq0$, and for all $\tilde x\in\mathbb R^d$, we have
\[
\int_{\mathbb R^d} \Phi_{a}^{b}\left(s,\frac{|x-\tilde x|^2}{2}\right)|x-\tilde x|^kdx<\infty.
\]
\end{lemma}
\begin{proof} We use Lemma \ref{Lemmaasymptotic} to deduce that
\begin{align*}
\begin{aligned}
&  \int_{\mathbb R^d}\Bigg[\Phi_a^b\left(s,\frac{|x-\tilde x|^2}{2}\right)|\tilde x-x|^k\Bigg]d x  
=   \int_0^\infty \Bigg[\Phi_a^b\left(s,\frac{r^2}{2}\right)r^k\Bigg]\frac{2\pi^{\frac{d}{2}}}{\Gamma(\frac{d}{2})}r^{d-1}dr\\
=&\frac{2\pi^{\frac{d}{2}}}{\Gamma(\frac{d}{2})}\int_0^\infty \Phi_a^b\left(s,\frac{r^2}{2}\right) r^{k+d-1}dr<\infty,
\end{aligned}
\end{align*}
where $\Gamma$ denotes the gamma function.
\end{proof}

\begin{lemma}\label{Lemmadecrease}
For any $0<a<b<\infty$, $s\in\mathbb R$, $z\in\mathbb R$, we have
\[
\frac{\partial}{\partial z}\Phi_{a}^{b}(s,z)=  -\Phi_{a}^{b}(s+1,z).
\]
In particular, for each $s\in\mathbb R$, $\Phi_{a}^{b}(s,\cdot)$ is infinitely differentiable, decreasing, and convex.
\end{lemma}
\begin{proof}
By Lebesgue's dominated convergence theorem we get
\[
\frac{d}{dz}\Phi_{a}^{b}(s,z)= \int_{\frac{1}{b^2}}^\frac{1}{a^2}\frac{d}{dz}\left(u^{s-1}e^{-zu}\right)du=\int_{\frac{1}{b^2}}^\frac{1}{a^2}\left(-u^{s}e^{-zu}\right)du=-\Phi_{a}^{b}(s+1,z).
\]
\end{proof}

\begin{lemma}\label{L3}
For any $0<b<\infty$, $s>0$ we have
\[
\lim_{a\searrow0}a^{2s}\Phi_{a}^{b}(s,0)=\frac{1}{s}.
\]
\end{lemma}
\begin{proof} A straightforward calculation yields
\[
\lim_{a\searrow0}\Phi_{a}^{b}(s,0)=\lim_{a\searrow0}\bigg(a^{2s}\int_{\frac{1}{b^2}}^\frac{1}{a^2} u^{s-1} du\bigg)= \lim_{a\searrow0}\frac{1}{s}\left(1-\frac{a^{2s}}{b^{2s}}\right)=\frac{1}{s}.
\]
\end{proof}

\begin{lemma}\label{L4}
For any $0<b<\infty$, $s>0$ and a function $z:[0,\infty)\to\mathbb R$ which is continuous and positive at zero,  we have
\[
\lim_{a\searrow0}\Phi_{a}^{b}(s,z(a))=\frac{1}{\big(z(0)\big)^s}\int_{\frac{z(0)}{b^2}}^{\infty} t^{s-1}e^{-t}dt<\infty.
\]
\end{lemma}
\begin{proof}
By Lebesgue's dominated convergence theorem we have
\[
\lim_{a\searrow0}\Phi_{a}^{b}(s,z(a))=\lim_{a\searrow0}  \frac{1}{(z(a))^s}\int_{\frac{z(a)}{b^2}}^{\frac{z(a)}{a^2}} t^{s-1}e^{-t}dt=\frac{1}{(z(0))^s}\int_{\frac{z(0)}{b^2}}^\infty t^{s-1}e^{-t}dt<\infty.
\]
\end{proof}

\begin{lemma}\label{L5}
For any $0<b<\infty$, $s>0$
and  $k>0$ we have
\[
\lim_{a\searrow0}a^{2s}\Phi_{a}^{b}(s,ka^2)=\frac{1}{k^s}\int_{0}^{k} t^{s-1}e^{-t}dt<\infty.
\]
\end{lemma}
\begin{proof}
By Lebesgue's dominated convergence theorem we have
\[
\lim_{a\searrow0}a^{2s}\Phi_{a}^{b}(s,ka^2)= \lim_{a\searrow0}\frac{1}{k^s}\int_{\frac{ka^2}{b^2}}^{k} t^{s-1}e^{-t}dt=\lim_{a\searrow0}\frac{1}{k^s}\int_{0}^{k} t^{s-1}e^{-t}dt<\infty.
\]
\end{proof}

\subsection{Auxiliary lemmas concerning $L_{a}^{b}$}
In this subsection we state and prove some basic properties of $L_{a}^{b}$ that is used throught the paper. Recall that $\mu_\mathrm{data}$ is a probability measure on $\mathbb R^d$, and for $0<a<b<\infty$, $s\in\mathbb R$, $x\in\mathbb R^d$, we define
\[
L_{a}^{b}(s,x):=\int_{\mathbb R^d}\Phi_a^b\left(s,\frac{|x-\tilde x|^2}{2}\right)d\mu_\mathrm{data}(\tilde x). 
\]
\begin{lemma}\label{Lemmacinfty}
For any given $0<a<b<\infty$, $s\in\mathbb R$, for any $d$-dimensional multi-index $\alpha$, we have
\[
\partial_x^\alpha L_{a}^{b}(s,x)=\int_{\mathbb R^d}\partial_x^\alpha \Phi_a^b\left(s,\frac{|x-\tilde x|^2}{2}\right)d\mu_\mathrm{data}(\tilde x).
\]
 In particular, $L_{a}^{b}(s,\cdot)$ is of class $C^\infty$.
\end{lemma}
\begin{proof}
By Lemma \ref{Lemmadecrease}, for any $\tilde x\in\mathbb R^d$ the function $x\mapsto\Phi_a^b(s,\frac{|x-\tilde x|^2}{2})$ has partial derivatives of all order, and they all have the form of $\Phi_a^b(s+k,\frac{|x-\tilde x|^2}{2})$ multiplied by a polynomial in $x-\tilde x$, for some nonnegative integer $k\geq0$. By Lemma \ref{Lemmaasymptotic}, each of these partial derivatives is uniformly bounded in $x\in\mathbb R^d$. Hence we can inductively apply the dominated convergence theorem to finish the proof.
\end{proof}

\begin{lemma}\label{Lemmafraction}
For $0<a<b<\infty$, $s_1,s_2\in\mathbb R$, and a compactly supported probability measure $\mu_\mathrm{data}$ on $\mathbb R^d$, we have
\[
\lim_{R\to\infty}\sup_{|x|\geq R}\left|\frac{ L_a^b\left(s_1,x\right) }{L_a^b\left(s_2,x\right)}-b^{2(s_2-s_1)}\right|=0,
\]
\end{lemma}
\begin{proof}
By Lemma \ref{Lemmaasymptotic}, we have
\[
\lim_{r\to\infty}\frac{ \Phi_{a}^{b}\left(s_1,r\right)}{ \Phi_{a}^{b}\left(s_2,r\right)} =\lim_{r\to\infty}\frac{ re^{\frac{r}{b^2}}\Phi_{a}^{b}\left(s_1,r\right)}{  re^{\frac{r}{b^2}}\Phi_{a}^{b}\left(s_2,r\right)} = \frac{b^{2s_2-2}}{b^{2s_1-2}} =b^{ 2(s_2-s_1)}
\]
Also note that since $\mu_\mathrm{data}$ is compactly supported, we have 
\[
\frac{|x-\tilde x|^2}{2}\geq\frac{( |x|-|\tilde x|)^2}{2}\geq \frac{( |x|-C)^2}{2}\quad\mbox{for}\quad |x|\geq C,~\tilde x\in\operatorname{supp}(\mu_\mathrm{data}),\]
with
\[ C:=\sup_{\tilde x\in\operatorname{supp}(\mu_\mathrm{data})}|\tilde x|<\infty.
\]
Hence $\frac{|x-\tilde x|^2}{2}\to\infty$ as $|x|\to\infty$, uniformly in $\tilde x\in\operatorname{supp}(\mu_\mathrm{data})$. Hence
\begin{align}\label{infsup}
\begin{aligned}
\lim_{R\to\infty}\inf_{|x|\geq R} \inf_{\tilde x\in\operatorname{supp}(\mu_\mathrm{data})}\frac{ \Phi_{a}^{b}\left(s_1,\frac{|x-\tilde x|^2}{2}\right)}{ \Phi_{a}^{b}\left(s_2,\frac{|x-\tilde x|^2}{2}\right)}
& =\lim_{R\to\infty}\sup_{|x|\geq R} \sup_{\tilde x\in\operatorname{supp}(\mu_\mathrm{data})}\frac{ \Phi_{a}^{b}\left(s_1,\frac{|x-\tilde x|^2}{2}\right)}{ \Phi_{a}^{b}\left(s_2,\frac{|x-\tilde x|^2}{2}\right)}  \\
&=b^{2(s_2-s_1)}.
\end{aligned}
\end{align}
On the other hand, from
\begin{align*}
\begin{aligned}
&\quad \frac{ L_a^b\left(s_1,x\right) }{L_a^b\left(s_2,x\right)}\int_{\mathbb R^d}\Phi_{a}^{b}\left(s_2,\frac{|x-\tilde x|^2}{2}\right)d\mu_\mathrm{data}(\tilde x)\\
&=\int_{\mathbb R^d}\left(\Phi_{a}^{b}\left(s_2,\frac{|x-\tilde x|^2}{2}\right)\times \frac{ \Phi_{a}^{b}\left(s_1,\frac{|x-\tilde x|^2}{2}\right)}{ \Phi_{a}^{b}\left(s_2,\frac{|x-\tilde x|^2}{2}\right)}\right)d\mu_\mathrm{data}(\tilde x),
\end{aligned}
\end{align*}
we get
\begin{equation}\label{fracineq}
\inf_{\tilde x\in\operatorname{supp}(\mu_\mathrm{data})} \frac{ \Phi_{a}^{b}\left(s_1,\frac{|x-\tilde x|^2}{2}\right)}{ \Phi_{a}^{b}\left(s_2,\frac{|x-\tilde x|^2}{2}\right)} 
\leq \frac{ L_a^b\left(s_1,x\right) }{L_a^b\left(s_2,x\right)}\leq\sup_{\tilde x\in\operatorname{supp}(\mu_\mathrm{data})} \frac{ \Phi_{a}^{b}\left(s_1,\frac{|x-\tilde x|^2}{2}\right)}{ \Phi_{a}^{b}\left(s_2,\frac{|x-\tilde x|^2}{2}\right)} 
\end{equation}
Now we combine \eqref{infsup} and \eqref{fracineq} to get
\begin{align*}
\begin{aligned}
&\quad \lim_{R\to\infty}\sup_{|x|\geq R}\left|\frac{ L_a^b\left(s_1,x\right) }{L_a^b\left(s_2,x\right)}-b^{2(s_2-s_1)}\right|\\
&= \lim_{R\to\infty}\sup_{|x|\geq R}\max\left\{\frac{ L_a^b\left(s_1,x\right) }{L_a^b\left(s_2,x\right)}-b^{2(s_2-s_1)},- \frac{ L_a^b\left(s_1,x\right) }{L_a^b\left(s_2,x\right)}+b^{2(s_2-s_1)}\right\}\\ 
&= \lim_{R\to\infty}\max\left\{\sup_{|x|\geq R}\left(\frac{ L_a^b\left(s_1,x\right) }{L_a^b\left(s_2,x\right)}-b^{2(s_2-s_1)}\right),-\inf_{|x|\geq R}\left(\frac{ L_a^b\left(s_1,x\right) }{L_a^b\left(s_2,x\right)}-b^{2(s_2-s_1)}\right)\right\}\\ 
&\leq\lim_{R\to\infty}\max\Bigg\{\sup_{|x|\geq R}\sup_{\tilde x\in\operatorname{supp}(\mu_\mathrm{data})}\frac{ \Phi_{a}^{b}\left(s_1,\frac{|x-\tilde x|^2}{2}\right)}{ \Phi_{a}^{b}\left(s_2,\frac{|x-\tilde x|^2}{2}\right)}-b^{2(s_2-s_1)},\\
&\qquad\qquad\qquad\qquad -\inf_{|x|\geq R}\inf_{\tilde x\in\operatorname{supp}(\mu_\mathrm{data})}\frac{ \Phi_{a}^{b}\left(s_1,\frac{|x-\tilde x|^2}{2}\right)}{ \Phi_{a}^{b}\left(s_2,\frac{|x-\tilde x|^2}{2}\right)}+b^{2(s_2-s_1)}\Bigg\}\\
&=\max\left\{b^{2(s_2-s_1)}-b^{2(s_2-s_1)},-b^{2(s_2-s_1)}+b^{2(s_2-s_1)}\right\}=0.
\end{aligned}
\end{align*}
Hence we have
\[
\lim_{R\to\infty}\sup_{|x|\geq R}\left|\frac{ L_a^b\left(s_1,x\right) }{L_a^b\left(s_2,x\right)}-b^{2(s_2-s_1)}\right|=0.
\]
\end{proof}

\begin{lemma}\label{Lemmacompact}
Assume that $\mu_{\mathrm{data}}$ is a compactly supported probability measure. Then for any $0<a<b<\infty$, $s\in\mathbb R$, and $0<c<\infty$, the superlevel set $\{x\in\mathbb R^d:L_a^b(s,x)\geq  c\}$ is compact.
\end{lemma}
\begin{proof}
Since the closedness of the superlevel set can be straightforwardly seen by the continuity of $L_a^b(s,\cdot)$, it only remains to show the boundedness. 
By Lemma \ref{Lemmaasymptotic},  there exists $R>0$ such that $z\geq R$ implies $\Phi_a^b(s,z)\leq \frac{c}{2}$. Also note that since $\mu_\mathrm{data}$ is compactly supported, we have
\[
C:=\sup_{\tilde x\in\operatorname{supp}(\mu_\mathrm{data})}|\tilde x|<\infty.
\]
Hence, for any $x\in\mathbb R^d$ satisfying $|x|\geq \sqrt{2R}+C$, we have $\Phi_a^b(s,\frac{|x-\tilde x|^2}{2})\leq \frac{c}{2}$ for all $\tilde x\in\operatorname{supp}(\mu_\mathrm{data})$, so that $L_a^b(s,x)\leq \frac{c}{2}$. From this we deduce that the superlevel set $\{x\in\mathbb R^d:L_a^b(s,x)\geq c\}$ is bounded, finishing the proof.
\end{proof}
%

\subsection{Proof of Proposition \ref{T1}}\label{T1proof}
We repeat the argument \eqref{fubini} for $\lambda(t)=(g(t))^2$ after we show that the three integrals in \eqref{3int} have finite values for a.e. $x\in\mathbb R^d$. Those integrals are indeed finite for a.e. $x\in\mathbb R^d$, as can be seen from the below calculations:
\begin{align}\label{denominator}
\begin{aligned}
\int_\varepsilon^T \sigma_t^{-2\alpha} \big(g(t)\big)^2  p_t(  x)dt=&\int_\varepsilon^T \sigma_t^{-2\alpha}\big(g(t)\big)^2 \Bigg(\int_{\mathbb R^d} \frac{e^{-\frac{| x- \tilde x|^2}{2\sigma_t^2}}}{( 2\pi\sigma_t^2)^{\frac{d}{2}}}d\mu_\mathrm{data}(\tilde x) \Bigg)dt\\
=&\frac{1}{( 2\pi )^{\frac{d}{2}}} \int_{\mathbb R^d }\int_\varepsilon^T  e^{-\frac{| x-  \tilde x |^2}{2\sigma_t^2}} \sigma_t^{-d-2\alpha}\big(g(t)\big)^2  dt d\mu_\mathrm{data}(\tilde x)\\
=&\frac{1}{( 2\pi )^{\frac{d}{2}}} \int_{\mathbb R^d} \int_{\frac{1}{(\sigma_T)^2}}^{\frac{1}{(\sigma_\varepsilon)^2}} e^{-\frac{|  x- \tilde  x|^2}{2 } u}    u^{\frac{d+2\alpha-4}{2}}  du d\mu_\mathrm{data}(\tilde x)\\
=& \frac{1}{( 2\pi )^{\frac{d}{2}}} \int_{\mathbb R^d} \Phi_{\sigma_\varepsilon}^{\sigma_T}\left(\frac{d+2\alpha-2}{2},\frac{|x-\tilde x|^2}{2}\right)d\mu_\mathrm{data}(\tilde x),
\end{aligned}
\end{align}

\begin{align}\label{numerator}
\begin{aligned}
\int_\varepsilon^T  \sigma_t^{-\alpha} \big(g(t)\big)^2\nabla   p_{\sigma_t}(  x)dt=&\int_\varepsilon^T \sigma_t^{-\alpha}\big(g(t)\big)^2 \Bigg(\int_{\mathbb R^d} \frac{e^{-\frac{| x- \tilde x |^2}{2\sigma_t^2}}}{( 2\pi\sigma_t^2)^{\frac{d}{2}}}\frac{\tilde x-x}{\sigma_t^2}d\mu_\mathrm{data}(\tilde x)\Bigg)dt\\
=&\frac{1}{( 2\pi )^{\frac{d}{2}}}\int_{\mathbb R^d}\int_\varepsilon^T  e^{-\frac{| x- \tilde x |^2}{2\sigma_t^2}} \sigma_t^{-d-\alpha -2}\big(g(t)\big)^2(\tilde x-x)  dtd\mu_\mathrm{data}(\tilde x)\\
=&\frac{1}{( 2\pi )^{\frac{d}{2}}}\int_{\mathbb R^d} \int_{\frac{1}{(\sigma_T)^2}}^{\frac{1}{(\sigma_\varepsilon)^2}} e^{-\frac{|  x- \tilde  x|^2}{2 } u}    u^{\frac{d+\alpha-2}{2}}(\tilde x-x)  dud\mu_\mathrm{data}(\tilde x) \\
=& \frac{1}{( 2\pi )^{\frac{d}{2}}} \int_{\mathbb R^d} \Phi_{\sigma_\varepsilon}^{\sigma_T}\left(\frac{d+\alpha}{2},\frac{|x-\tilde x|^2}{2}\right)(\tilde x-x)d\mu_\mathrm{data}(\tilde x),
\end{aligned}
\end{align}
and
\begin{align}\label{finite}
\begin{aligned}
&\int_\varepsilon^T \big(g(t)\big)^2\frac{|\nabla  p_{\sigma_t}(  x)|^2}{  p_{\sigma_t}(  x)}dt\\
=& \int_\varepsilon^T  \big(g(t)\big)^2 \Bigg|\int_{\mathbb R^d} \frac{e^{-\frac{| x-  \tilde x |^2}{2\sigma_t^2}}}{( 2\pi\sigma_t^2)^{\frac{d}{2}}}\frac{\tilde x-x}{\sigma_t^2}d\mu_\mathrm{data}(\tilde x) \Bigg|^2\Big/\Bigg(\int_{\mathbb R^d} \frac{e^{-\frac{| x-  \tilde x |^2}{2\sigma_t^2}}}{( 2\pi\sigma_t^2)^{\frac{d}{2}}}d\mu_\mathrm{data}(\tilde x)\Bigg) dt\\
\leq& \int_\varepsilon^T  \big(g(t)\big)^2 \Bigg(\int_{\mathbb R^d} \frac{e^{-\frac{| x- \tilde x |^2}{2\sigma_t^2}}}{( 2\pi\sigma_t^2)^{\frac{d}{2}}}\frac{|\tilde x-x|^2}{\sigma_t^4}d\mu_\mathrm{data}(\tilde x)\Bigg) dt\\
=&\frac{1}{( 2\pi )^{\frac{d}{2}}}\int_{\mathbb R^d}\int_\varepsilon^T  e^{-\frac{| x- \tilde x |^2}{2\sigma_t^2}} \sigma_t^{-d -4 }\big(g(t)\big)^2 |\tilde x-x|^2dtd\mu_\mathrm{data}(\tilde x)\\
=&\frac{1}{( 2\pi )^{\frac{d}{2}}} \int_{\mathbb R^d}\int_{\frac{1}{(\sigma_T)^2}}^{\frac{1}{(\sigma_\varepsilon)^2}} e^{-\frac{|  x- \tilde x|^2}{2 } u}    u^{\frac{d }{2}} |\tilde x-x|^2 du d\mu_\mathrm{data}(\tilde x)\\
=&\frac{1}{( 2\pi )^{\frac{d}{2}}}\int_{\mathbb R^d}\Phi_{\sigma_\varepsilon}^{\sigma_T}\left(\frac{d+2}{2},\frac{|x-\tilde x|^2}{2}\right)|\tilde x-x|^2d\mu_\mathrm{data}(\tilde x)\\
\leq&\frac{2}{( 2\pi )^{\frac{d}{2}}}\sup_{z\geq0}\left[z\Phi_{\sigma_\varepsilon}^{\sigma_T}\left(\frac{d+2}{2},z\right)\right]<\infty,
\end{aligned}
\end{align}
where we used the Cauchy-Schwarz inequality and Lemma \ref{Lemmaasymptotic} in the first and the last inequality, respectively. We use \eqref{denominator}, \eqref{numerator}, \eqref{finite} to get
\begin{align*}
\begin{aligned}
&\int_{\mathbb R^d}\Bigg[\int_\varepsilon^T |\sigma_t^{-\alpha}s_\theta(  x)-\nabla \log   p_{\sigma_t}(  x)|^2  p_{\sigma_t}(  x)\big(g(t)\big)^2dt\Bigg]d  x\\
=&\int_{\mathbb R^d}\Bigg[\bigg(\int_\varepsilon^T\sigma_t^{-2\alpha}\big(g(t)\big)^2  p_{\sigma_t}(  x)dt\bigg)  \Bigg|s_\theta(  x)-\frac{\int_\varepsilon^T\sigma_t^{-\alpha}\big(g(t)\big)^2 \nabla p_{\sigma_t}(  x)dt}{\int_\varepsilon^T\sigma_t^{-2\alpha}\big(g(t)\big)^2  p_{\sigma_t}(  x)dt}\Bigg|^2\\
&\qquad-\frac{\big|\int_\varepsilon^T\sigma_t^{-\alpha}\big(g(t)\big)^2\nabla  p_{\sigma_t}(  x)dt\big|^2}{\int_\varepsilon^T\sigma_t^{-2\alpha}\big(g(t)\big)^2  p_{\sigma_t}(  x)dt} +  \bigg(\int_\varepsilon^T \big(g(t)\big)^2\frac{|\nabla  p_{\sigma_t}(  x)|^2}{  p_{\sigma_t}(  x)}dt\bigg)\Bigg] d  x\\
=&\int_{\mathbb R^d}\Bigg[\bigg(
  \frac{1}{( 2\pi )^{\frac{d}{2}}} \int_{\mathbb R^d}\Phi_{\sigma_\varepsilon}^{\sigma_T}\left(\frac{d+2\alpha-2}{2},\frac{|x-\tilde x|^2}{2}\right)d\mu_\mathrm{data}(\tilde x)\bigg)   |s_\theta(  x)-s_*(x) |^2\\
&\qquad-\frac{\big|\int_\varepsilon^T\sigma_t^{-\alpha}\big(g(t)\big)^2\nabla  p_{\sigma_t}(  x)dt\big|^2}{\int_\varepsilon^T\sigma_t^{-2\alpha}\big(g(t)\big)^2  p_{\sigma_t}(  x)dt} +  \bigg(\int_\varepsilon^T \big(g(t)\big)^2\frac{|\nabla  p_{\sigma_t}(  x)|^2}{  p_{\sigma_t}(  x)}dt\bigg)\Bigg] d  x.
\end{aligned}
\end{align*}
Now the following estimate of the constant $C$ finishes the proof:

\begin{align*}
\begin{aligned}
C&=\int_{\mathbb R^d}\Bigg[-\frac{\big|\int_\varepsilon^T\sigma_t^{-\alpha}\big(g(t)\big)^2\nabla  p_{\sigma_t}(  x)dt\big|^2}{\int_\varepsilon^T\sigma_t^{-2\alpha}\big(g(t)\big)^2  p_{\sigma_t}(  x)dt} +  \bigg(\int_\varepsilon^T \big(g(t)\big)^2\frac{|\nabla  p_{\sigma_t}(  x)|^2}{  p_{\sigma_t}(  x)}dt\bigg)\Bigg] d  x\\
&\leq \int_{\mathbb R^d}\Bigg[0+\frac{1}{( 2\pi )^{\frac{d}{2}}}\int_{\mathbb R^d}\Phi_{\sigma_\varepsilon}^{\sigma_T}\left(\frac{d+2}{2},\frac{|x-\tilde x|^2}{2}\right)|\tilde x-x|^2d\mu_\mathrm{data}(\tilde x)\Bigg]d x \\
&= \frac{1}{( 2\pi )^{\frac{d}{2}}}\int_{\mathbb R^d}\int_{\mathbb R^d}\Bigg[\Phi_{\sigma_\varepsilon}^{\sigma_T}\left(\frac{d+2}{2},\frac{|x-\tilde x|^2}{2}\right)|\tilde x-x|^2\Bigg]d x d\mu_\mathrm{data}(\tilde x)<\infty,
\end{aligned}
\end{align*}
where in the last inequality we used Lemma \ref{Lemmaconstant}.

\subsection{Proof of Proposition \ref{Theoremlimsup}}\label{Prooflimsup}

We decompose $\frac{|s_*(x)|}{|x|}$ as follows:
\begin{align*}
\begin{aligned}
\frac{|s_*(x)|}{|x|}&=\frac{\left|\int_{\mathbb R^d}\Phi_{\sigma_\varepsilon}^{\sigma_T}\left(\frac{d+\alpha }{2},\frac{|x-\tilde x|^2}{2}\right)(\tilde x-x)d\mu_\mathrm{data}(\tilde x)\right|}{\int_{\mathbb R^d}\Phi_{\sigma_\varepsilon}^{\sigma_T}\left(\frac{d+2\alpha-2}{2},\frac{|x-\tilde x|^2}{2}\right)|x|d\mu_\mathrm{data}(\tilde x)}\\
&=\frac{ \int_{\mathbb R^d}\Phi_{\sigma_\varepsilon}^{\sigma_T}\left(\frac{d+\alpha }{2},\frac{|x-\tilde x|^2}{2}\right)d\mu_\mathrm{data}(\tilde x) }{\int_{\mathbb R^d}\Phi_{\sigma_\varepsilon}^{\sigma_T}\left(\frac{d+2\alpha-2}{2},\frac{|x-\tilde x|^2}{2}\right)d\mu_\mathrm{data}(\tilde x)}\times\frac{\left|\int_{\mathbb R^d}\Phi_{\sigma_\varepsilon}^{\sigma_T}\left(\frac{d+\alpha }{2},\frac{|x-\tilde x|^2}{2}\right)(\tilde x-x)d\mu_\mathrm{data}(\tilde x)\right|}{\int_{\mathbb R^d}\Phi_{\sigma_\varepsilon}^{\sigma_T}\left(\frac{d+\alpha}{2},\frac{|x-\tilde x|^2}{2}\right)|x|d\mu_\mathrm{data}(\tilde x)}\\
&=:\mathcal I_1\times\mathcal I_2.
\end{aligned}
\end{align*}
We compute the uniform limits of $\mathcal I_1$ and $\mathcal I_2$ as $|x|\to\infty$ separately, and then combine them to get the desired result.\\

$\bullet$ Uniform limit of $\mathcal I_1$ as $|x|\to\infty$ : By Lemma \ref{Lemmafraction}, we have
\begin{equation}\label{i1}
\lim_{R\to\infty}\sup_{|x|\geq R}|\mathcal I_1-\sigma_T^{\alpha-2}|=0.
\end{equation}

$\bullet$ Uniform limit of $\mathcal I_2$ as $|x|\to\infty$ : We bound $|\mathcal I_2-1|$ as
\begin{align*}
\begin{aligned}
|\mathcal I_2-1|&=\left|\frac{|\mathcal I_{21}|}{|\mathcal I_{22}|}-1\right|\leq\frac{|\mathcal I_{21}+\mathcal I_{22}|}{|\mathcal I_{22}|}=\frac{\left|\int_{\mathbb R^d}\Phi_{\sigma_\varepsilon}^{\sigma_T}\left(\frac{d+\alpha }{2},\frac{|x-\tilde x|^2}{2}\right)\tilde x\right|d\mu_\mathrm{data}(\tilde x)}{\int_{\mathbb R^d}\Phi_{\sigma_\varepsilon}^{\sigma_T}\left(\frac{d+\alpha}{2},\frac{|x-\tilde x|^2}{2}\right)|x|d\mu_\mathrm{data}(\tilde x)}\\
&\leq \frac{ \int_{\mathbb R^d}\Phi_{\sigma_\varepsilon}^{\sigma_T}\left(\frac{d+\alpha }{2},\frac{|x-\tilde x|^2}{2}\right)|\tilde x|d\mu_\mathrm{data}(\tilde x) }{\int_{\mathbb R^d}\Phi_{\sigma_\varepsilon}^{\sigma_T}\left(\frac{d+\alpha}{2},\frac{|x-\tilde x|^2}{2}\right)|x|d\mu_\mathrm{data}(\tilde x)}\leq \frac{C}{|x|},
\end{aligned}
\end{align*}
where
\begin{align*}
\begin{aligned}
&\mathcal I_{21}:=\int_{\mathbb R^d}\Phi_{\sigma_\varepsilon}^{\sigma_T}\left(\frac{d+\alpha }{2},\frac{|x-\tilde x|^2}{2}\right)(\tilde x-x)d\mu_\mathrm{data}(\tilde x),\\
&\mathcal I_{22}:=\int_{\mathbb R^d}\Phi_{\sigma_\varepsilon}^{\sigma_T}\left(\frac{d+\alpha}{2},\frac{|x-\tilde x|^2}{2}\right)xd\mu_\mathrm{data}(\tilde x),
\end{aligned}
\end{align*}
and
\[
C:=\sup_{\tilde x\in\operatorname{supp}(\mu_\mathrm{data})}|\tilde x|<\infty.
\]
By taking the uniform limit of both sides as $|x|\to\infty$, we have
\[
\lim_{R\to\infty}\sup_{|x|\geq R}|\mathcal I_2-1|\leq \lim_{R\to\infty}\sup_{|x|\geq R}\frac{C}{|x|}=\lim_{R\to\infty}\frac{C}{R}=0.
\]
Hence
\begin{equation}\label{i2}
\lim_{R\to\infty}\sup_{|x|\geq R}|\mathcal I_2-1|=0.
\end{equation}

$\bullet$ Uniform limit of $\frac{|s_*(x)|}{|x|}$ as $|x|\to\infty$ : We combine \eqref{i1}, \eqref{i2} and the estimate
\begin{align*}
\begin{aligned}
&\sup_{|x|\geq R}\left|\frac{|s_*(x)|}{|x|}-\sigma_T^{\alpha-2}\right|=\sup_{|x|\geq R}\left|\mathcal I_1 \mathcal I_2-\sigma_T^{\alpha-2}\right|\\
=&\sup_{|x|\geq R} |(\mathcal I_1-\sigma_T^{\alpha-2}) (\mathcal I_2-1)+(\mathcal I_1-\sigma_T^{\alpha-2} )+\sigma_T^{\alpha-2}(\mathcal I_2-1)|\\
\leq&\bigg(\sup_{|x|\geq R} |\mathcal I_1-\sigma_T^{\alpha-2}|\bigg)\bigg(\sup_{|x|\geq R}|\mathcal I_2-1|\bigg)+\bigg(\sup_{|x|\geq R}|\mathcal I_1-\sigma_T^{\alpha-2} |\bigg)+\sigma_T^{\alpha-2}\bigg(\sup_{|x|\geq R}|\mathcal I_2-1|\bigg)
\end{aligned}
\end{align*}
to deduce that
\[
\lim_{R\to\infty}\sup_{|x|\geq R}\left|\frac{|s_*(x)|}{|x|}-\sigma_T^{\alpha-2}\right|=0,
\]
finishing the proof.

\subsection{Proof of Theorem \ref{Theoremexistence}}\label{Proofexistence}
Since $s_*+e$ is locally Lipschitz, a solution $x(t)$ uniquely exists on some open interval containing $t=0$. By Proposition \ref{Theoremlimsup}, $|s_*+e|$ satisfies the following linear growth condition:
\[
C:=\sup_{x\in\mathbb R^d}\frac{|s_*(x)+e(x)|}{1+|x|}<\infty.
\]
Let $(a,b)$ be the maximal interval of existence. 
From
\[
x(t)=x_0+\int_0^t (s_*(x(s))+e(x(s)))ds,~~t\in (a,b),
\]
we have
\[
|x(t)|\leq |x_0|+ \int_0^{t}C(1+|x(s)|)ds,~~t\in[0,b)
\]
and
\[
|x(t)|\leq |x_0|+\int_t^{0}C(1+|x(s)|)ds,~~t\in(a,0],
\]
or equivalently,
\[
|x(-|t|)|\leq |x_0|+\int_0^{|t|}C(1+|x(-s)|)ds,~~t\in(a,0].
\]
By the Gr\"onwall inequality, we get
\[
|x(t)|\leq (|x_0|+C|t|)e^{C|t|},~~ t\in(a,b).
\]
Hence the solution does not blow up in finite time, i.e., we have $a=-\infty$ and $b=\infty$.

\subsection{Proof of Theorem \ref{Theoremattract}}\label{Proofattract}
Before we prove Theorem \ref{Theoremattract}, we introduce a proposition that provides us a way to construct the attracting set, given a candidate fundamental neighborhood.
\begin{proposition}\label{Propattract} \cite{Ruelle}
Let $\{\phi_t\}_{t\in\mathbb R}$ be a dynamical system on $\mathbb{R}^d$. If $U\subset\mathbb R^d$ is open, $\phi_t(U)$ is contained in $U$ and relatively compact for sufficiently large $t$, then
\[
\bigcap_{t\geq 0}\phi_t(U)
\]
is a compact attracting set with fundamental neighborhood $U$.
\end{proposition}
Now we are ready to prove Theorem \ref{Theoremattract}.
For convenience, we use the notation 
\[
L(\cdot):=L_{\sigma_\varepsilon}^{\sigma_T}\left(\frac{d+\alpha}{2}-1,\cdot\right)
\]
and
\[
\tilde L(\cdot):=L_{\sigma_\varepsilon}^{\sigma_T}\left(\frac{d+2\alpha}{2}-1,\cdot\right).
\] 
Denote by $\phi_t$ the flow generated by the system \eqref{mainsystem}.\\
 
$\bullet$ Step 1: For any $x\in\mathbb R^d$, the flow $\phi_t(x)$ satisfies  
\[
\frac{d}{dt}L\left(\phi_t(x)\right)=\tilde L\left(\phi_t(x)\right)\left\langle s_*(\phi_t(x)),s_*(\phi_t(x))+e(\phi_t(x))\right\rangle.
\]
{\it Proof of Step 1.} We use Lemma \ref{Lemmacinfty} to get
\[
s_*(x)=\frac{\nabla_x L(x)}{\tilde L(x)}.
\]
We use this to compute
\[
\frac{d}{dt}L\left(\phi_t(x)\right)=\langle\nabla L(\phi_t(x)),s_*(\phi_t(x))+e(\phi_t(x))\rangle=\tilde L\left(\phi_t(x)\right)\left\langle s_*(\phi_t(x)),s_*(\phi_t(x))+e(\phi_t(x))\right\rangle.
\]

$\bullet$ Step 2: $\phi_t(A)\subset A$ for all $t\geq0$.\\
{\it Proof of Step 2.} Let $x\in A$ be arbitrarily given. Assume the contrary and suppose that
\[
\{t\geq 0:\phi_t(x)\notin A\}\neq\varnothing.
\]
Since $A$ is a closed set contained in the open set $\Gamma$, by connectedness of the trajectory $\phi_t(x)$ we also have
\[
S:=\{t>0:\phi_t(x)\in\Gamma\setminus A\}\neq\varnothing.
\]
Choose $(t_1,t_2]\in S$ satisfying $\phi_{t_1}(x)\in A$. We combine Step 1 and $(\Gamma\setminus A)\cap E=\varnothing$ to get 
\[
\frac{d}{dt}L\left(\phi_t(x)\right)>0, \quad t_1<t<t_2.
\]
Hence $b>L(\phi_{t_2}(x))> L(\phi_{t_1}(x))=b$, a contradiction. Therefore we have $\phi_t(x)\in A$ for all $t\geq0$.\\ 

$\bullet$ Step 3: $\phi_t(U_{c})\subset U_{c}$ for all $t\geq0$, $c\in(a,b)$, where
\[
U_c:=\{x\in\Gamma: L(x)>c\}.
\]
{\it Proof of Step 3.} Fix $c\in(a,b)$. By Step 2, it suffices to show that 
\[\phi_t(U_c\setminus A)\subset U_c,\quad\forall~t\geq0.
\] 
Let $x\in U_c$ be arbitrarily given. Define 
\[t_3:=\sup\{t>0:\phi_s(x)\in U_c\setminus A~\mbox{for}~0\leq s<t\}.
\] By the continuity of $\phi_t(x)$, we have $t_3\in(0,\infty]$. We consider the two cases $t_3=\infty$ and $t_3<\infty$ separately.

(i) $t_3=\infty$: By definition, we have $\phi_t(x)\in U_c$ for all $t\geq0$.

(ii) $t_3<\infty$: We have either $L(\phi_{t_3}(x))=c$ or $b$. We combine Step 1 and $(\Gamma\setminus A)\cap E=\varnothing$ to get 
\[
\frac{d}{dt}L\left(\phi_t(x)\right)>0, \quad 0<t<t_3.
\]
Hence $L(\phi_{t_3}(x))> L(x)>c$. From this, we get $L(\phi_{t_3}(x))=b$ and $\phi_{t_3}(x)\in  A$. By Step 2, we have $\phi_t(x)\in U_c$ for all $t\geq0$.

Combining (i) and (ii) completes the proof.\\

$\bullet$ Step 4: For any $(c_1,c_2)\subset(a,b)$, $\phi_t(U_{c_1})\subset U_{c_2}$ for sufficiently large $t\geq0$.\\
{\it Proof of Step 4.} Let $x\in U_{c_1}$ be arbitrarily given. Define 
\[t_4:=\sup\{t>0:\phi_s(x)\in U_{c_1}\setminus U_{c_2}~\mbox{for}~0\leq s<t\}.
\] By continuity of $\phi_t(x)$, we have $t_4\in(0,\infty]$. By Step 1 we have
\[
\frac{d}{dt}L\left(\phi_t(x)\right)\geq m>0,\quad 0<t<t_4,
\]
where 
\[
m:=\inf_{z\in K}\tilde L(z)\left(|s_*(z)|^2-\langle s_*(z),e(z)\rangle\right),\quad K:=\{z\in\Gamma:c_1\leq L(z)\leq c_2\},
\]
and we have used $(\Gamma\setminus A)\cap E=\varnothing$ and the compactness of $K$ (Lemma \ref{Lemmacompact}) to get $m>0$.
Hence we have
\[
c_2\geq L(\phi_t(x))> c_1+mt,\quad 0\leq t<t_4.
\]
Therefore $t_4\leq \frac{c_2-c_1}{m}$, and $\phi_t(x)\in U_{c_2}$ for all $t>t_4$, by Step 3. In particular, we have $\phi_t(x)\in U_{c_2}$ for all $t>\frac{c_2-c_1}{m}$. Since 
$\frac{c_2-c_1}{m}$ does not depend on the choice of $x\in U_{c_1}$, we have $\phi_t(U_{c_1})\subset U_{c_2}$ for all $t>\frac{c_2-c_1}{m}$.\\

$\bullet$ Step 5: For any $c\in(a,b)$, $U_c$ is a fundamental neighborhood of the compact attracting set defined as
\[
\Lambda:=\bigcap_{t\geq0}\phi_t(U_c),
\]
which does not depend on the choice of $c\in(a,b)$, and is a subset of $A$.\\
{\it Proof of Step 5.} By Step 4, $\Lambda$ does not depend on $c\in(a,b)$.
By Proposition \ref{Propattract} and Step 4, $\Lambda$ is the compact attracting set corresponding to each of the fundamental neighborhoods $U_c$ $(c\in(a,b))$. By Step 4, we have
\[
\Lambda\subset \bigcap_{d\in(a,b)}U_d=A.
\]

$\bullet$ Step 6: The basin of attraction of $\Lambda$ contains $\Gamma$.\\
{\it Proof of Step 6.} By Step 5, the basin of attraction contains the set
\[
\bigcup_{c\in(a,b)}U_c=\Gamma.
\]

\subsection{Proof of Theorem \ref{Theoremlasalle}}\label{Prooflasalle}
Before we prove Theorem \ref{Theoremlasalle}, we recall LaSalle's invariance principle \cite{LaSalle}.
\begin{proposition}\label{Proplasalle}\cite{LaSalle}
Let a locally Lipschitz continuous function $f:\mathbb R^d\to\mathbb R^d$ be given. Let $\phi_t$ be the flow governed by the autonomous system $\dot x=f(x)$. Let a bounded open set $G\subset \mathbb R^d$ be given. Suppose that $G$ is positively invariant, i.e., $\phi_t(G)\subset G$ for all $t\geq0$. Let a $C^1$ function $V:\mathbb R^d\to\mathbb R$ satisfy
\[
\dot V(x)=\langle\nabla V(x),f(x)\rangle\leq 0,\quad x\in G.
\]
Define
\[
E=\{x\in\bar G: \dot V(x)=0\}
\]
and let $M$ be the union of all solutions that remain in $E$ on their maximal interval of existence. Suppose $M \subset G$. Then $\phi_t(x)\to M$ as $t\to\infty$ for all $x\in G$.
\end{proposition}
{\it Proof of Theorem \ref{Theoremlasalle}.} For convenience, we use the notation 
\[
L(\cdot):=L_{\sigma_\varepsilon}^{\sigma_T}\left(\frac{d+\alpha}{2}-1,\cdot\right),\quad
\tilde L(\cdot):=L_{\sigma_\varepsilon}^{\sigma_T}\left(\frac{d+2\alpha}{2}-1,\cdot\right).
\]
We may express $s_*(x)=\frac{\nabla L(x)}{\tilde L(x)}$ (Lemma \ref{Lemmacinfty}), and $E=\{x\in\mathbb R^d: s_*(x)=0\}$. By Proposition \ref{Theoremlimsup}, $E$ is compact, so it is contained in $G_\delta$ for sufficiently small $\delta>0$, where
\[
G_\delta:=\{x\in\mathbb R^d: L(x)>\delta\},\quad \delta>0.
\]
Each $G_\delta$ is open, bounded (Lemma \ref{Lemmacompact}), and positively invariant as can be seen by the following observation: 
\begin{equation*}
-\dot L(x)=-\frac{|\nabla L(x)|^2}{\tilde L(x)}\leq 0,~~x\in\mathbb R^d.
\end{equation*}
For sufficiently small $\delta>0$, we have
\[
E=\{x\in\bar G_\delta: -\dot L(x)=0\}.
\]
Note that for each $x\in E$, we have $s_*(x)=0$ and hence $\phi_t(x)=x$ for all $t\in\mathbb R$. By Proposition \ref{Proplasalle}, for sufficiently small $\delta>0$ we have $\phi_t(x)\to E$ as $t\to\infty$ for all $x\in G_\delta$. Since $\bigcup_{0<\delta<\delta_0}G_\delta=\mathbb R^d$ for any $\delta_0>0$, the theorem follows.

\subsection{Proof of Theorem \ref{Theoremoverfit}}\label{Proofoverfit}
The following proposition introduces a classical method using Liapunov functions to study asymptotic behaviors of dynamical systems. 
\begin{proposition}\label{Propclassic}
Let an open set $U\subset\mathbb R^d$ and a $C^1$ function $f:U\to\mathbb R^d$ be given. Let $x_0\in  U$ be an equilibrium point of the dynamical system
\[
\dot x=f(x),
\]
i.e., $f(x_0)=0$. Suppose that there exists a $C^1$ function $V:U\to\mathbb R$ such that $V(x_0)=0$, $V(x)>0$ for $x\neq x_0$, and $\nabla V(x)\cdot f(x)<0$ for $x\neq x_0$. Then $x_0$ is asymptotically stable.
\end{proposition}
$V$ satisfying the hypothesis of the above proposition is called a {\it Liapunov function.} For more detailed discussion on this subject, see \cite{Perko, Wiggins}.
Throughout the proof, we use the notation 
\[
L(\cdot):=L_{\sigma_\varepsilon}^{\sigma_T}\left(\frac{d+\alpha}{2}-1,\cdot\right).
\]


$\bullet$ Step 1: $L(x)$ is strictly concave on $|x-x_i|< k\sigma_\varepsilon$ for sufficiently small $\sigma_\varepsilon>0$.\\
{\it Proof of Step 1.} It suffices to show that the Hessian of $L$, denoted by $\operatorname{Hess}_L$, is negative definite on the region $|x-x_i|< k\sigma_\varepsilon$ for sufficiently small $\sigma_\varepsilon>0$. We use Lemma \ref{Lemmadecrease} to compute
\[
\nabla L(x)= \frac{1}{N}\sum_{j=1}^N\Phi_{\sigma_\varepsilon}^{\sigma_T}\left(\frac{d+\alpha}{2},\frac{|x-x_j|^2}{2}\right)(x_j-x)
\]
and
\begin{align*}
\begin{aligned}
 \operatorname{Hess}_L(x)&= \frac{1}{N}\sum_{j=1}^N\Phi_{\sigma_\varepsilon}^{\sigma_T}\left(\frac{d+\alpha+2}{2},\frac{|x-x_j|^2}{2}\right) (x_j-x)   (x_j-x)^\top\\
&\quad- \frac{1}{N}\sum_{j=1}^N\Phi_{\sigma_\varepsilon}^{\sigma_T}\left(\frac{d+\alpha}{2},\frac{|x-x_j|^2}{2}\right)I.
\end{aligned}
\end{align*}
For $|z|=1$, we have
\begin{align}\label{concave0}
\begin{aligned}
 z^\top(\operatorname{Hess}_L(x))z &\leq \frac{1}{N}\sum_{j=1}^N\Phi_{\sigma_\varepsilon}^{\sigma_T}\left(\frac{d+\alpha+2}{2},\frac{|x-x_j|^2}{2}\right) |x-x_j|^2\\
&\quad- \frac{1}{N}\sum_{j=1}^N\Phi_{\sigma_\varepsilon}^{\sigma_T}\left(\frac{d+\alpha }{2},\frac{|x-x_j|^2}{2}\right).
\end{aligned}
\end{align}
By convexity of $z\mapsto\Phi_{\sigma_\varepsilon}^{\sigma_T}(\frac{d+\alpha}{2},z)$ (Lemma \ref{Lemmadecrease}), we have
\begin{align}\label{concave1}
\begin{aligned}
\Phi_{\sigma_\varepsilon}^{\sigma_T}\left(\frac{d+\alpha }{2},\frac{|x-x_i|^2}{2}\right)&\geq  \Phi_{\sigma_\varepsilon}^{\sigma_T}\left(\frac{d+\alpha }{2},0\right)+\frac{\partial\Phi_{\sigma_\varepsilon}^{\sigma_T}}{\partial z}\left(\frac{d+\alpha }{2},0\right)\frac{|x-x_i|^2}{2}\\
&\geq  \Phi_{\sigma_\varepsilon}^{\sigma_T}\left(\frac{d+\alpha }{2},0\right)-\Phi_{\sigma_\varepsilon}^{\sigma_T}\left(\frac{d+\alpha+2 }{2},0\right)\frac{k^2\sigma_\varepsilon^2}{2}.
\end{aligned}
\end{align}
Since $z\mapsto\Phi_{\sigma_\varepsilon}^{\sigma_T}(s,z)$ is decreasing for any fixed $s\in\mathbb R$ (Lemma \ref{Lemmadecrease}), we have
\begin{equation}\label{concave2}
\Phi_{\sigma_\varepsilon}^{\sigma_T}\left(\frac{d+\alpha+2}{2},\frac{|x-x_i|^2}{2}\right) |x-x_i|^2\leq \Phi_{\sigma_\varepsilon}^{\sigma_T}\left(\frac{d+\alpha+2}{2},0\right) k^2\sigma_\varepsilon^2
\end{equation}
and for all $j\neq i$,  we use $\big||x_i-x_j|-|x-x_i|\big|\leq|x-x_j|\leq |x_i-x_j|+|x-x_i|$ to get
\begin{align}\label{concave3}
\begin{aligned}
&\quad \Phi_{\sigma_\varepsilon}^{\sigma_T}\left(\frac{d+\alpha+2}{2},\frac{|x-x_j|^2}{2}\right) |x-x_j|^2\\
&\leq \Phi_{\sigma_\varepsilon}^{\sigma_T}\left(\frac{d+\alpha+2}{2},\frac{\big||x_i-x_j|-k\sigma_\varepsilon\big|^2}{2}\right) (|x_i-x_j|+k\sigma_\varepsilon)^2
\end{aligned}
\end{align}
and
\begin{equation}\label{concave4}
\Phi_{\sigma_\varepsilon}^{\sigma_T}\left(\frac{d+\alpha}{2},\frac{|x-x_j|^2}{2}\right)\leq \Phi_{\sigma_\varepsilon}^{\sigma_T}\left(\frac{d+\alpha}{2},\frac{\big||x_i-x_j|-k\sigma_\varepsilon\big|^2}{2}\right).
\end{equation}
By combining \eqref{concave0}, \eqref{concave1}, \eqref{concave2}, \eqref{concave3} and \eqref{concave4}, we have the following for $|z|=1$:
\begin{align*}
\begin{aligned}
Nz^\top(\operatorname{Hess}_L(x))z &\leq  -\Phi_{\sigma_\varepsilon}^{\sigma_T}\left(\frac{d+\alpha }{2},0\right)+\frac{3}{2}\Phi_{\sigma_\varepsilon}^{\sigma_T}\left(\frac{d+\alpha+2}{2},0\right) k^2\sigma_\varepsilon^2\\
&\quad+\sum_{j\neq i}\Phi_{\sigma_\varepsilon}^{\sigma_T}\left(\frac{d+\alpha+2}{2},\frac{\big||x_i-x_j|-k\sigma_\varepsilon\big|^2}{2}\right) (|x_i-x_j|+k\sigma_\varepsilon)^2\\
&\quad+\sum_{j\neq i}\Phi_{\sigma_\varepsilon}^{\sigma_T}\left(\frac{d+\alpha}{2},\frac{\big||x_i-x_j|-k\sigma_\varepsilon\big|^2}{2}\right) \\
&=:\mathcal J_1.
\end{aligned}
\end{align*}
Note that by Lemma \ref{L3} and Lemma \ref{L4}, we have
\[
\lim_{\sigma_\varepsilon\searrow0}\sigma_\varepsilon^{d+\alpha}\mathcal J_1=-\frac{2}{d+\alpha}+\frac{3k^2}{d+\alpha+2}<0.
\]
Hence there exist some constant $\delta_1(d,\alpha,\sigma_T,k, \mathcal X)$ such that $0<\varepsilon<\delta_1(d,\alpha,\sigma_T,k, \mathcal X)$ implies $\sigma_\varepsilon^{d+\alpha}\mathcal J_1<0$, or equivalently, $\mathcal J_1<0$. In this case, $\operatorname{Hess}_L(x)$ is negative definite on the region $|x-x_i|<k\sigma_\varepsilon$, and therefore $L$ is strictly concave on that region.\\

$\bullet$ Step 2: $L(x_i)> L(x)$ on $|x-x_i|=k\sigma_\varepsilon$ for sufficiently small $\sigma_\varepsilon>0$.\\

{\it Proof of Step 2.} Note that
\begin{align}\label{sphere0}
\begin{aligned}
&\quad L(x_i)-L(x)\\
&= \frac{1}{N}\sum_{j=1}^N\left(\Phi_{\sigma_\varepsilon}^{\sigma_T}\left(\frac{d+\alpha-2}{2},\frac{|x_i-x_j|^2}{2}\right)-\Phi_{\sigma_\varepsilon}^{\sigma_T}\left(\frac{d+\alpha-2}{2},\frac{|x-x_j|^2}{2}\right)\right).
\end{aligned}
\end{align}
By convexity of $z\mapsto \Phi_{\sigma_\varepsilon}^{\sigma_T}(\frac{d+\alpha-2}{2},z)$ (Lemma \ref{Lemmadecrease}) we have
\begin{align}\label{sphere1}
\begin{aligned}
&\Phi_{\sigma_\varepsilon}^{\sigma_T}\left(\frac{d+\alpha-2}{2},0\right)-\Phi_{\sigma_\varepsilon}^{\sigma_T}\left(\frac{d+\alpha-2}{2},\frac{|x-x_i|^2}{2}\right)\\
\geq&\frac{\partial\Phi_{\sigma_\varepsilon}^{\sigma_T}}{\partial z}\left(\frac{d+\alpha-2}{2},\frac{|x-x_i|^2}{2}\right)\left(0-\frac{|x-x_i|^2}{2}\right) \\
= &\Phi_{\sigma_\varepsilon}^{\sigma_T} \left(\frac{d+\alpha}{2},\frac{|x-x_i|^2}{2}\right) \frac{|x-x_i|^2}{2} =\Phi_{\sigma_\varepsilon}^{\sigma_T}\left(\frac{d+\alpha}{2},\frac{k^2\sigma_\varepsilon^2}{2}\right)\frac{k^2\sigma_\varepsilon^2}{2},
\end{aligned}
\end{align}
and for $j\neq i$ we use $\big||x_i-x_j|-|x-x_i|\big|\leq|x-x_j|$ to get
\begin{align}\label{sphere2}
\begin{aligned}
&\Phi_{\sigma_\varepsilon}^{\sigma_T}\left(\frac{d+\alpha-2}{2},\frac{|x_i-x_j|^2}{2}\right)-\Phi_{\sigma_\varepsilon}^{\sigma_T}\left(\frac{d+\alpha-2}{2},\frac{|x-x_j|^2}{2}\right)\\
\geq& \frac{\partial\Phi_{\sigma_\varepsilon}^{\sigma_T}}{\partial z}\left(\frac{d+\alpha-2}{2},\frac{|x-x_j|^2}{2}\right)\left(\frac{|x_i-x_j|^2}{2}-\frac{|x-x_j|^2}{2}\right)\\
=& \Phi_{\sigma_\varepsilon}^{\sigma_T}\left(\frac{d+\alpha}{2},\frac{|x-x_j|^2}{2}\right)\left(\frac{|x-x_j|^2}{2}-\frac{|x_i-x_j|^2}{2}\right)\\
\geq& \Phi_{\sigma_\varepsilon}^{\sigma_T}\left(\frac{d+\alpha}{2},\frac{|x-x_j|^2}{2}\right)\bigg(\frac{\big||x_i-x_j|-k\sigma_\varepsilon\big|^2}{2}-\frac{|x_i-x_j|^2}{2}\bigg)\\
\geq& \Phi_{\sigma_\varepsilon}^{\sigma_T}\left(\frac{d+\alpha}{2},\frac{\big||x_i-x_j|-k\sigma_\varepsilon\big|^2}{2}\right)\left( -k\sigma_\varepsilon|x_i-x_j|  \right),
\end{aligned}
\end{align}
where the monotonicity of $\Phi_{\sigma_\varepsilon}^{\sigma_T}$ (Lemma \ref{Lemmadecrease}) was used in the last inequality. We combine \eqref{sphere0}, \eqref{sphere1}, \eqref{sphere2} to deduce that
\begin{align*}
\begin{aligned}
N&(L(x_i)-L(x))\\
&\geq\Phi_{\sigma_\varepsilon}^{\sigma_T}\left(\frac{d+\alpha}{2},\frac{k^2\sigma_\varepsilon^2}{2}\right)\frac{k^2\sigma_\varepsilon^2}{2}-\sum_{j\neq i}\Phi_{\sigma_\varepsilon}^{\sigma_T}\left(\frac{d+\alpha}{2},\frac{\big||x_i-x_j|-k\sigma_\varepsilon\big|^2}{2}\right) k\sigma_\varepsilon|x_i-x_j|\\
&  =:\mathcal J_2.
\end{aligned}
\end{align*}
By Lemma \ref{L4} and Lemma \ref{L5} we have
\[
\lim_{\varepsilon\searrow0}\sigma_\varepsilon^{d+\alpha-2} \mathcal J_2=\frac{2^{\frac{d+\alpha-2}{2}}}{k^{d+\alpha-2}}\int_{0}^{\frac{k^2}{2}} t^{s-1}e^{-t}dt>0.
\]
Hence there exist some constant $\delta_2(d,\alpha,\sigma_T,k, \mathcal X)$ such that $0<\sigma_\varepsilon<\delta_2(d,\alpha,\sigma_T,k, \mathcal X)$ implies $\sigma_\varepsilon^{d+\alpha-2}\mathcal J_2>0$, or equivalently, $\mathcal J_2>0$. In this case, $L(x_i)> L(x)$ on the sphere $|x-x_i|=k\sigma_\varepsilon$.\\

$\bullet$ Step 3: For sufficiently small $\sigma_\varepsilon>0$, $L$ restricted to the open ball  $|x-x_i|<k\sigma_\varepsilon$ has a unique global maximizer $\hat x_i$, which is also a unique stationary point in that region.\\
{\it Proof of Step 3.} Assume that $\sigma_\varepsilon$ is sufficiently small that it satisfies Step 1 and Step 2. Since $L$ is continuous on the compact set $|x-x_i|\leq k\sigma_\varepsilon$, the restriction of $L$ to this compact set attains its maximum, say at $\hat x_i$. By Step 2, the maximum is attained in the interior of the ball, i.e., $|\hat x_i-x_i|<k\sigma_\varepsilon$. Since $\hat x_i$ is a global maximizer in the interior of the domain, it is also a local maximizer, and is a stationary point. Uniqueness of a global maximizer and a stationary point is clear from Step 1.\\

$\bullet$ Step 4: $\hat x_i$ in Step 3 is a unique equilibrium point in the open ball  $U=|x-x_i|<k\sigma_\varepsilon$, and is asymptotically stable. \\
{\it Proof of Step 4.} Set $V(x):=L(\hat x_i)-L(x)$. By uniqueness of the global maximizer $\hat x_i$ of $L|_{U}$, we have $V(\hat x_i)=0$, $V(x)>0$ for $x\in U\setminus \{\hat x_i\}$, and by uniqueness of the stationary point $\hat x_i$ of $L$, we have $s_*(\hat x_i)=0$ and
\[
\nabla V(x)\cdot s_*(x)
=- \frac{|\nabla L(x)|^2}{\frac{1}{N}\sum_{j=1}^N\Phi_{\sigma_\varepsilon}^{\sigma_T}\left(\frac{d+2\alpha-2}{2},\frac{|x-x_j|^2}{2}\right)}
<~0,\quad x\in U\setminus\{\hat x_i\}.
\]
By Proposition \ref{Propclassic}, $\hat x_i$ is asymptotically stable. Uniqueness of the equilibrium point follows from that of the stationary point of $L$.

\bibliographystyle{plainnat}
\bibliography{Kode}
\end{document}